\documentclass{article}
\usepackage{iclr2027_conference,times}
\usepackage[T1]{fontenc}
\usepackage{amsmath,amssymb,amsthm,mathtools}
\usepackage{graphicx,booktabs,multirow,array}
\usepackage{xcolor,tikz}
\usetikzlibrary{arrows.meta,positioning,calc,fit}
\usepackage{algorithm,algpseudocode}
\usepackage[section]{placeins}
\usepackage{titletoc}
\usepackage{hyperref,url}
\usepackage{enumitem}
\hypersetup{hidelinks}

\newtheorem{theorem}{Theorem}[section]
\newtheorem{lemma}[theorem]{Lemma}
\newtheorem{proposition}[theorem]{Proposition}

\newtheorem{assumption}[theorem]{Assumption}
\theoremstyle{remark}

\title{HO-FL: Hybrid-Order Federated Learning\\for Heterogeneous Edge Devices}
\author{%
\textbf{Qiyuan Chen\textsuperscript{1,}\thanks{Equal contribution.}\,, Xian Wu\textsuperscript{1,*}, Yanan Ma\textsuperscript{2}, Xianhao Chen\textsuperscript{1,}\thanks{Corresponding author: Xianhao Chen (\texttt{xcheneee@hku.hk}).}}\\
\normalfont\textsuperscript{1}University of Hong Kong (hku.hk)\\
\normalfont\textsuperscript{2}City University of Hong Kong (cityu.edu.hk)\\
\normalfont\texttt{qiyuanchen@connect.hku.hk, u3012772@connect.hku.hk}\\
\normalfont\texttt{yananma8-c@my.cityu.edu.hk, xcheneee@hku.hk}%
}
\begin{document}
\iclrfinalcopy
\maketitle
\begin{abstract}
Federated learning (FL) on memory-constrained edge devices faces a dilemma: first-order (FO) optimization (i.e., backpropagation) demands substantial memory, whereas zeroth-order (ZO) optimization suffers from severe convergence slowdown. To resolve this dilemma, we introduce HO-FL, a hybrid-order FL framework that trains a model's bottom segment with ZO optimization and its top segment with FO optimization. Each device can flexibly select its order boundary according to its memory budget while participating in the training of the same global model. Moreover, our convergence analysis reveals a new, fundamental trade-off: clients with larger FO-trained segments can provide more accurate updates, but favoring them can underrepresent other clients' data. We connect this trade-off to the bias and variance of actual multi-step local updates, yielding a sampling optimization problem and a practical dimension-aware approximation with direct model averaging. Experiments on language tasks examine task performance, client memory, and sampling under data heterogeneity. The results show that hybrid-order local training can retain much of the full-FO performance with substantially lower client memory requirements.
Our code is available at \url{https://github.com/HKU-WILL-Lab/HO-FL}.
\end{abstract}
\section{Introduction}
\label{sec:intro}

Deploying and fine-tuning large language models (LLMs) directly on edge devices provides a promising paradigm for unlocking decentralized real-world data while preserving user data privacy \citep{wu2025survey}.
Federated learning (FL) \citep{fedavg} has been widely explored as an effective framework to collaboratively adapt models across distributed clients without centralizing raw data.
However, standard FL relies on first-order (FO) optimization via backpropagation (BP), which requires storing intermediate activations across model layers.
For modern LLMs, these activations easily exceed the physical memory limits of commodity edge hardware, posing a severe memory bottleneck for on-device federated adaptation.

To mitigate this activation memory overhead, zeroth-order (ZO) optimization has been increasingly adopted in federated edge learning \citep{fedzo,decomfl}.
By estimating gradient directions solely through perturbed forward passes, ZO optimization avoids constructing a backward computation graph, thereby reducing client training memory to the inference level.
Nevertheless, pure ZO gradient optimization incurs a severe variance penalty: random directional perturbations produce gradient estimators whose variance scales proportionally with the trainable parameter dimension $d$.
Consequently, pure ZO federated algorithms typically require orders-of-magnitude more local iterations and communication rounds than their FO counterparts to reach target performance \citep{hosfl,hiso}, limiting their practicality for complex downstream tasks.

Recent efforts on hybrid-order (HO) optimization provide an appealing methodology to overcome this memory--accuracy dilemma.
By partitioning the network architecture, HO optimization structurally decouples the optimization landscape: the bottom layers are optimized via BP-free ZO perturbations, while the top layers are trained using FO gradients.
While early exploratory efforts have leveraged such hybrid updates within split learning paradigms \citep{hosfl,hosl,heron}, directly adapting hybrid-order optimization to practical FL is hindered by heterogeneous edge constraints.
In federated settings, real-world edge devices exhibit diverse physical memory capacities; enforcing a globally identical order boundary either exceeds the memory budget of resource-constrained devices or underutilizes the capacity of more capable ones.

These challenges lead to a natural question: \textit{Can we design a fully federated hybrid-order optimization paradigm that accommodates heterogeneous edge hardware while retaining the convergence advantages of first-order training?}

To answer this question, we propose \textbf{HO-FL}, a hybrid-order FL framework tailored for heterogeneous edge devices.
In HO-FL, each client executes both ZO updates on the bottom segment and FO backpropagation on the top segment entirely on-device, performing multi-step local training and synchronizing with the central server only at the end of each round.
To accommodate hardware disparity, HO-FL enables each client to configure a personalized order boundary matched to its local memory budget.
To alleviate uplink communication overhead, HO-FL further encodes bottom ZO updates into compact scalar-seed pairs, requiring full-vector uploads only for the top FO parameters.

We establish the non-convex convergence of HO-FL, proving that the average gradient norm is bounded by $\mathcal{O}\big(\sqrt{(1 + \frac{1}{q}\sum_{i=1}^N p_i d_i^{\mathrm{ZO}})/T}\big) + \mathcal{O}\big(\zeta^2 N \sum_{i=1}^N (p_i - \frac{1}{N})^2\big)$.
Here, the first term characterizes the optimization rate governed by the sampling-weighted ZO parameter dimension, while the second term quantifies the representation bias induced by client data heterogeneity $\zeta^2$.
Furthermore, our bound provides a unified theoretical guarantee, seamlessly recovering existing hybrid-order and federated convergence rates as special cases when clients participate uniformly ($p_i = 1/N$) or share homogeneous order boundaries.

Crucially, this bound reveals a promising optimization opportunity: the two terms formalize a fundamental trade-off between local update quality and global data representation.
While uniform sampling is unavoidably bottlenecked by noisy updates from ZO-dominated clients, non-uniform selection $p$ can actively suppress gradient variance by prioritizing FO-dominant clients at the cost of a controlled representation penalty.
Grounded in this trade-off, we design an efficient \textbf{Dimension-Aware Client Sampling} policy via dependent rounding, provably and empirically outperforms uniform sampling.
Our main contributions are summarized as follows:
\begin{itemize}[leftmargin=20pt]

    \item \textbf{Framework \& Unified Convergence Theory:} We propose HO-FL, the first fully federated hybrid-order optimization framework supporting client-adaptive order boundaries and compact seed-based uplink transmission. We establish a unified convergence theory for HO-FL, showing that the convergence bound decomposes into an optimization term $\mathcal{O}\big(\sqrt{(1 + \frac{1}{q}\sum_{i=1}^N p_i d_i^{\mathrm{ZO}})/T}\big)$ and a sampling bias term $\mathcal{O}\big(\zeta^2 N \sum_{i=1}^N (p_i - \frac{1}{N})^2\big)$, seamlessly subsuming existing hybrid-order analyses as special cases.
    \item \textbf{Dimension-Aware Client Sampling:} Grounded in our convergence theory, we formulate a dimension-aware client sampling policy via dependent rounding. Our strategy selectively prioritizes low-variance updates while bounding sampling bias, which is demonstrated both theoretically and empirically to accelerate global convergence over uniform client sampling.
    \item \textbf{Empirical Validation:} We evaluate HO-FL across diverse language understanding and generation benchmarks on OPT-125M, Qwen2.5-1.5B, and SmolLM3-3B. The results demonstrate that HO-FL closely approaches full first-order federated performance while slashing client peak memory by up to 77.4\% on the most resource-constrained devices and substantially outperforming pure ZO baselines.
\end{itemize}
\section{Related Work}
\label{sec:related}

\paragraph{Zeroth-order optimization in federated learning.}
Zeroth-order (ZO) optimization avoids activation caching by estimating gradients via randomized finite differences, reducing client memory to inference levels~\citep{mezo,fedzo}.
To curb the communication overhead of exchanging perturbations, DeComFL~\citep{decomfl} and HiSo~\citep{hiso} employ pseudo-random seed synchronization and Hessian guidance, respectively.
Nevertheless, pure ZO federated algorithms inevitably suffer from gradient estimation variance that scales with the full parameter dimension $d$, severely slowing convergence.
HO-FL circumvents this bottleneck by confining ZO optimization strictly to memory-critical bottom segments.

\paragraph{Hybrid-order optimization across distributed paradigms.}
To balance memory footprint and gradient accuracy, recent works partition networks into a bottom ZO segment and a top FO segment. 
Although existing methods realize this hybrid structure differently, they can be cast into a unified framework governed by boundary activations and localized surrogate objectives (see Appendix~\ref{app:unified-framework} for details): HOSL~\citep{hosl} evaluates perturbations through end-to-end forward passes; HO-SFL~\citep{hosfl} freezes cut-layer activation feedback to construct a linear surrogate; and HERON-SFL~\citep{heron} introduces a client-side auxiliary network to decouple perturbations from server dependencies. 
However, these methods are fundamentally tied to split learning (SL) architectures, which require per-iteration activation transmissions and enforce identical order boundaries across all clients. 
In contrast, HO-FL executes both segments fully locally, allows client-adaptive order boundaries to fit their heterogeneous memory limits, and enables dimension-aware client sampling to enhance performance.

\section{HO-FL: Hybrid-Order Federated Optimization}
\label{sec:method}

\paragraph{Problem formulation.}
We consider an FL system with $N$ distributed edge clients, where each client $i \in \{1, \dots, N\}$ holds a local data distribution $\mathcal{D}_i$. Given global model parameters $w \in \mathbb{R}^d$ and a sample loss $\ell(w; x, y)$, the global objective is:
\begin{equation}
  \min_{w \in \mathbb{R}^d} f(w) \coloneqq \frac{1}{N} \sum_{i=1}^N f_i(w), \qquad
  \text{where} \quad f_i(w) \coloneqq \mathbb{E}_{(x, y) \sim \mathcal{D}_i} \left[ \ell(w; x, y) \right].
  \label{eq:objective}
\end{equation}

\paragraph{Client sampling via dependent rounding.}
At each global round $t$, the server samples a cohort $S_t$ of $K$ distinct clients among $N$ clients to participate in training. 
We characterize client participation via normalized selection probabilities $p_i$, which reflect the long-term proportion of selections allocated to client $i$ over the cumulative sampling budget, satisfying $p_i \in [0, 1/K]$ and $\sum_{i=1}^N p_i = 1$. 
While multiple sampling schemes can realize this marginal distribution, we draw the $K$ distinct participants via dependent rounding~\citep{depround}, implemented as $S_t = \operatorname{DepRound}(K p)$ in Algorithm~\ref{alg:depround}. This choice ensures theoretical tractability, with structural properties detailed in Appendix~\ref{app:round-variance}.
The server then dispatches independent pseudorandom seeds $\{\mathrm{seed}_{i,t,j}^{e}\}_{e=0, j=1}^{E-1, q}$ to each selected client $i \in S_t$ for subsequent local ZO updates.

\begin{figure}[t]
\centering
\input{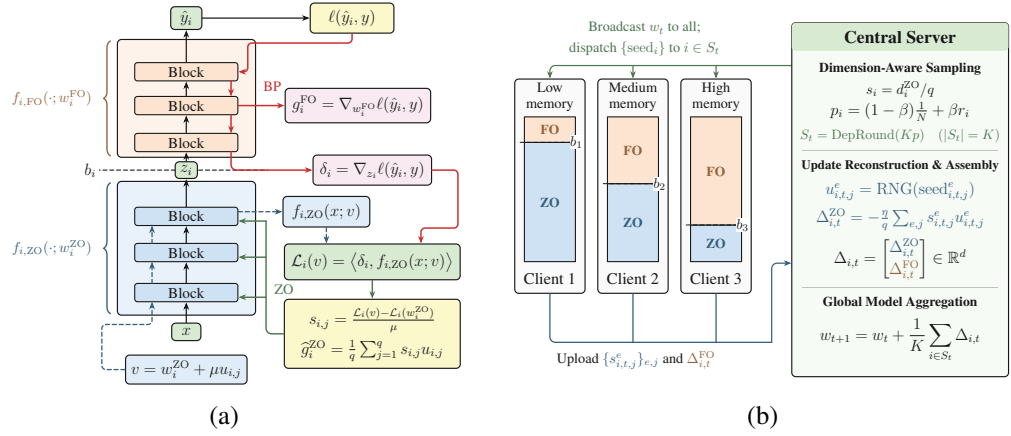}
\caption{\textbf{Overview of the HO-FL framework.} (a) \textit{Client-side hybrid-order step}: the bottom segment $f_{i,\mathrm{ZO}}$ is updated via ZO perturbations evaluated on surrogate objective $\mathcal{L}_i$ guided by order-boundary activation feedback $\delta_i$, while the top segment $f_{i,\mathrm{FO}}$ performs standard BP. (b) \textit{System-level federated workflow}: the server selects cohorts via dimension-aware dependent rounding, broadcasts $w_t$, and reconstructs full model updates in $\mathbb{R}^d$ via seed replay for global model aggregation.}
\label{fig:overview}
% Separate panel references without changing the template's caption style.
\begingroup
\makeatletter
\protected@edef\@currentlabel{\thefigure(a)}\label{fig:overview-local}
\protected@edef\@currentlabel{\thefigure(b)}\label{fig:overview-system}
\makeatother
\endgroup
\end{figure}

\paragraph{Memory-adaptive boundaries.}
Client $i$ chooses a boundary $b_i$ according to its memory constraint and partitioning its trainable parameters into two parts: 
\begin{equation}
  w_i = \left[ (w_i^{\mathrm{ZO}})^\top, (w_i^{\mathrm{FO}})^\top \right]^\top,
  \label{eq:parameter-blocks}
\end{equation}
where $w_i^{\mathrm{ZO}} \in \mathbb{R}^{d_i^{\mathrm{ZO}}}$ and $w_i^{\mathrm{FO}} \in \mathbb{R}^{d_i^{\mathrm{FO}}}$ denote the bottom and top parameter segments, respectively, satisfying $d_i^{\mathrm{ZO}} + d_i^{\mathrm{FO}} = d$.
Since clients can update the ZO segment without caching the corresponding activation, flexible order boundaries allow resource-constrained devices to participate by expanding $w_i^{\mathrm{ZO}}$, while enabling more capable devices to retain larger FO segments.

\paragraph{A hybrid-order local step.}
As illustrated in Figure~\ref{fig:overview-local}, we represent the neural network on client $i$ via a bottom sub-network $f_{i,\mathrm{ZO}}(\cdot; w_i^{\mathrm{ZO}})$ and a top sub-network $f_{i,\mathrm{FO}}(\cdot; w_i^{\mathrm{FO}})$. For a minibatch $\xi=(x,y)$, the bottom segment first computes the boundary activation $z_i = f_{i,\mathrm{ZO}}(x; w_i^{\mathrm{ZO}})$ without caching intermediate activations. 
Feeding $z_i$ into the top segment yields the task prediction $\hat{y}_i = f_{i,\mathrm{FO}}(z_i; w_i^{\mathrm{FO}})$ and its associated loss $\ell(\hat{y}_i,y)$. Subsequently, backpropagation through the top segment produces both the parameter gradient with respect to $w_i^{\mathrm{FO}}$ and the activation gradient with respect to $z_i$: 
\begin{equation}
 g_i^{\mathrm{FO}}=\nabla_{w_i^{\mathrm{FO}}}\ell(w_i;\xi),\qquad
 \delta_i=\nabla_{z_i}\ell(f_{i,\mathrm{FO}}(z_i;w_i^{\mathrm{FO}}),y).
 \label{eq:top-feedback}
\end{equation}
Following HO-SFL~\citep{hosfl}, we construct a local surrogate objective for the bottom parameters $v$:
\begin{equation}
 \mathcal L_i(v)=\langle\delta_i,f_{i,\mathrm{ZO}}(x;v)\rangle.
 \label{eq:surrogate}
\end{equation}
By the chain rule, the gradient of this surrogate objective with respect to $v$ exactly recovers the first-order gradient of the original loss. Introducing this surrogate eliminates the need for end-to-end forward passes with perturbed parameters, restricting ZO evaluations strictly to the bottom segment.

Using pseudorandom seeds dispatched by the server, the client generates $q$ independent random perturbation vectors via a random number generator (RNG):
\begin{equation}
 u_{i,j}=\operatorname{RNG}(\mathrm{seed}_{i,j})
 \sim\mathcal N(0,I_{d_i^{\mathrm{ZO}}}),\qquad j=1,\ldots,q.
 \label{eq:seed-direction}
\end{equation}
The client then estimates the ZO gradient for the bottom segment using simultaneous perturbation stochastic approximation (SPSA):
\begin{equation}
  s_{i,j} = \frac{\mathcal{L}_i(w_i^{\mathrm{ZO}} + \mu u_{i,j}) - \mathcal{L}_i(w_i^{\mathrm{ZO}})}{\mu}, \qquad
  \widehat{g}_i^{\mathrm{ZO}} = \frac{1}{q} \sum_{j=1}^{q} s_{i,j} u_{i,j},
  \label{eq:local-zo}
\end{equation}
In our framework, each client performs this hybrid-order update for $E$ local steps per communication round. Algorithm~\ref{alg:client-update} combines the forward pass, top-segment BP, and bottom-segment estimation in a single local loop. The boundary gradient $\delta_i$ is recomputed at each step and held fixed within its $q$ finite differences. Indexing the global round by $t$ and the local iteration by $e$, both segment gradients are evaluated at $w_{i,t}^{e}$ before the updates:
\begin{equation}
  \begin{aligned}
    w_{i,t}^{\mathrm{ZO},e+1} &= w_{i,t}^{\mathrm{ZO},e} - \eta \widehat{g}_{i,t}^{\mathrm{ZO},e}, \\
    w_{i,t}^{\mathrm{FO},e+1} &= w_{i,t}^{\mathrm{FO},e} - \eta g_{i,t}^{\mathrm{FO},e},
  \end{aligned}
  \qquad e = 0, \dots, E-1.
  \label{eq:local-updates}
\end{equation}

\paragraph{Client uploads.}
Upon completing $E$ local steps, client $i$ uploads only the sequence of directional projection scalars $\{s_{i,t,j}^{e}\}$ and the accumulated top-segment parameter difference:
\begin{equation}
  \Delta_{i,t}^{\mathrm{FO}} = w_{i,t}^{\mathrm{FO},E} - w_{i,t}^{\mathrm{FO},0}.
  \label{eq:top-difference}
\end{equation}
Because the random seeds are pre-assigned by the server, the client does not need to transmit the high-dimensional perturbation vectors. Compared to conventional FL where full model deltas in $\mathbb{R}^d$ must be uploaded, the ZO component in HO-FL achieves dimension-free uplink transmission. As a consequence, configuring a larger ZO segment not only alleviates peak activation memory on the client, but also substantially reduces uplink communication overhead.

\paragraph{Server reconstruction and aggregation.}
Upon collecting the local updates from the selected participant set $S_t$, the server calls \textsc{Reconstruct} (Algorithm~\ref{alg:reconstruct}) to recover the bottom-segment parameter updates using the known seeds:
\begin{equation}
  u_{i,t,j}^{e} = \operatorname{RNG}(\mathrm{seed}_{i,t,j}^{e}), \qquad
  \Delta_{i,t}^{\mathrm{ZO}} = -\frac{\eta}{q} \sum_{e=0}^{E-1} \sum_{j=1}^{q} s_{i,t,j}^{e} u_{i,t,j}^{e}.
  \label{eq:bottom-reconstruction}
\end{equation}
The server then concatenates $\Delta_{i,t}^{\mathrm{ZO}}$ and the received $\Delta_{i,t}^{\mathrm{FO}}$ according to client $i$'s order boundary $b_i$, reconstructing the full model update $\Delta_{i,t} = [(\Delta_{i,t}^{\mathrm{ZO}})^\top, (\Delta_{i,t}^{\mathrm{FO}})^\top]^\top \in \mathbb{R}^d$. Although participating devices adopt heterogeneous order boundaries, their updates are fully compatible in the shared parameter space, enabling direct coordinate-wise model averaging:
\begin{equation}
  w_{t+1} = w_t + \frac{1}{K} \sum_{i \in S_t} \Delta_{i,t}.
  \label{eq:aggregation-step}
\end{equation}
The updated model $w_{t+1}$ is subsequently broadcast to clients for the next round. Under this design, the client uplink transmission per round requires only $\mathcal{O}(Eq + d_i^{\mathrm{FO}})$ elements, which scales gracefully even for overparameterized architectures. For stateful optimizers such as AdamW, the server exactly mirrors the client trajectory via deterministic replay, as elaborated in Appendix~\ref{app:adamw-replay}.

% Main algorithm: gradient updates match the theory and Figure 1.
% Shared algorithm styling. Palette follows the current Overleaf Figure 1.
% Load after algorithm, algpseudocode, xcolor, and hyperref.
\ifdefined\HOFLAlgorithmStyleLoaded\else
\def\HOFLAlgorithmStyleLoaded{1}
\definecolor{hoAlgZO}{HTML}{244A78}
\definecolor{hoAlgFO}{HTML}{D98732}
\definecolor{hoAlgServer}{HTML}{39816B}
\definecolor{hoAlgGray}{HTML}{65717D}
\newcommand{\HOFunc}[2]{\textcolor{#1}{\textbf{\textsc{#2}}}}
\newcommand{\HOCall}[4]{\hyperref[#2]{\HOFunc{#1}{#3}}\ensuremath{(#4)}}
\newcommand{\HOComment}[1]{\Comment{\textcolor{hoAlgGray}{\scriptsize #1}}}
\newcommand{\HOPhase}[1]{\Statex\textcolor{hoAlgGray}{\textit{// #1}}}
\fi

\begin{algorithm}[!htbp]
\caption{HO-FL: hybrid-order federated optimization}
\label{alg:hofl}
\begin{algorithmic}[1]
\Require $w_0$; $\{b_i,\mathcal D_i\}_{i=1}^{N}$; $T,K,E,q,\mu,\eta,\beta$.
\State Set $r$ by~\eqref{eq:practical-policy}.
\State $p_i\gets(1-\beta)/N+\beta r_i$, $i=1,\ldots,N$.
\For{$t=0,\ldots,T-1$}
  \HOPhase{Server: sampling and broadcast}
  \State $S_t\gets$ \HOCall{hoAlgServer}{alg:depround}{DepRound}{Kp}.
    \HOComment{Algorithm~\ref{alg:depround}}
  \State Broadcast $w_t$; dispatch independent $\{\mathrm{seed}_{i,t,j}^{e}\}_{e,j}$ to each $i\in S_t$.
  \HOPhase{Clients: local hybrid-order training}
  \For{each $i\in S_t$ \textbf{in parallel}}
    \State $(\{s_{i,t,j}^{e}\}_{e,j},\Delta_{i,t}^{\mathrm{FO}})\gets$
      \HOCall{hoAlgZO}{alg:client-update}{ClientUpdate}{i,t,w_t}.
      \HOComment{Algorithm~\ref{alg:client-update}}
    \State Upload \textcolor{hoAlgZO}{$\{s_{i,t,j}^{e}\}_{e,j}$} and
      \textcolor{hoAlgFO}{$\Delta_{i,t}^{\mathrm{FO}}$}.
  \EndFor
  \HOPhase{Server: reconstruction and direct averaging}
  \For{each $i\in S_t$}
    \State $\Delta_{i,t}\gets$
      \HOCall{hoAlgServer}{alg:reconstruct}{Reconstruct}{i,t,\{s_{i,t,j}^{e}\}_{e,j},\Delta_{i,t}^{\mathrm{FO}}}.
      \HOComment{Algorithm~\ref{alg:reconstruct}}
  \EndFor
  \State \textcolor{hoAlgServer}{$w_{t+1}\gets w_t+\frac1K\sum_{i\in S_t}\Delta_{i,t}$}.
    \HOComment{Eq.~\eqref{eq:aggregation-step}}
\EndFor
\State \Return $w_T$.
\end{algorithmic}
\end{algorithm}

% ===== BEGIN sections/theory_v3.tex =====
\section{Convergence Analysis}
\label{sec:theory}

In this section, we analyze the theoretical convergence of the HO-FL local updates in~\eqref{eq:local-updates} under non-convex objectives, establishing the theoretical foundation for our sampling policy in Section~\ref{sec:sampling}.

\subsection{Theoretical Assumptions}
\label{sec:theory-assumptions}
Our analysis builds upon standard regularity conditions in non-convex distributed and zeroth-order optimization~\citep{scaffold,hosfl}. 
Specifically, we assume that each client objective $f_i$ is $L$-smooth and the global loss is bounded below (Assumption~\ref{ass:objective}); reference stochastic gradients are conditionally unbiased with bounded variance (Assumption~\ref{ass:reference}); the local surrogate $\mathcal{L}_i$ is $L_i^s$-smooth (Assumption~\ref{ass:surrogate-new}); and inter-client data divergence satisfies the affine heterogeneity condition with residual heterogeneity bound $\zeta^2$ (Assumption~\ref{ass:affine-heterogeneity}). 
Formal mathematical statements and regularity discussions are deferred to Appendix~\ref{app:conditions}.

\subsection{Non-Convex Convergence Rate}
\label{sec:theory-convergence}
By controlling the multi-step local trajectory drift and hybrid gradient estimation variance (with intermediate descent guarantees deferred to Appendix~\ref{app:main-proof}), and properly calibrating the step size and perturbation radius, we establish the following convergence rate of HO-FL.

\begin{theorem}[Informal Statement of Non-Convex Convergence]
\label{thm:convergence}
Under Assumptions~\ref{ass:objective}--\ref{ass:affine-heterogeneity}, with the local step size $\eta$ and perturbation radius $\mu$ suitably calibrated, the iterates of HO-FL satisfy:
\begin{equation}
  \frac{1}{T}\sum_{t=0}^{T-1}\mathbb E\|\nabla f(w_t)\|^2
  \le
  \underbrace{\mathcal O\!\left(\sqrt{\frac{1+\frac{1}{q}\sum_{i=1}^N p_i d_i^{\mathrm{ZO}}}{T}}\right)}_{\text{Optimization Rate}}
  +\underbrace{\mathcal O\!\left(\zeta^2N\sum_{i=1}^N\left(p_i-\frac{1}{N}\right)^2\right)}_{\text{Sampling Representation Bias}}.
  \label{eq:main-explicit-rate}
\end{equation}
\end{theorem}
Here, $\zeta^2$ represents the client data heterogeneity parameter defined in Assumption~\ref{ass:affine-heterogeneity}. The formal statement with explicit step-size conditions and the complete proof is provided in Appendix~\ref{app:affine}.

\paragraph{Theoretical generality.}
Theorem~\ref{thm:convergence} provides a unified theoretical guarantee that subsumes existing hybrid-order paradigms as special cases. 
Specifically, under uniform client participation ($p_i = 1/N$) and a global identical order boundary ($d_i^{\mathrm{ZO}} \equiv d^{\mathrm{ZO}}$), the sampling representation bias vanishes, reducing the bound in~\eqref{eq:main-explicit-rate} to:
\begin{equation}
  \frac{1}{T}\sum_{t=0}^{T-1}\mathbb E\|\nabla f(w_t)\|^2
  \le \mathcal O\!\left(\sqrt{\frac{1 + d^{\mathrm{ZO}}/q}{T}}\right).
  \label{eq:rate-ho-recovery}
\end{equation}
This rate exactly recovers the non-convex convergence established in prior hybrid-order frameworks such as HO-SFL~\citep{hosfl}. 
Furthermore, at the architectural extremes, Eq.~\eqref{eq:rate-ho-recovery} naturally collapses to the canonical $\mathcal{O}(1/\sqrt{T})$ rate of FedAvg~\citep{wang2023fedavg} for full first-order training ($d^{\mathrm{ZO}} = 0$), and the $\mathcal{O}(\sqrt{(1 + d/q)/T})$ rate of pure ZO-FL~\citep{fedzo,decomfl} when updates are entirely zeroth-order ($d^{\mathrm{ZO}} = d$). 
This demonstrates that our convergence bound offers a strictly more general theoretical framework.

% ===== END sections/theory_v3.tex =====

% \paragraph{Optimization opportunity via non-uniform sampling.}
% Crucially, our bound exposes a profound optimization opportunity enabled by non-uniform sampling across heterogeneous edge devices. 
% Under uniform client selection, the optimization rate in~\eqref{eq:rate-uniform-reduction} is unavoidably bottlenecked by the unweighted arithmetic average $\frac{1}{N}\sum_i d_i^{\mathrm{ZO}}$, which is heavily degraded by severely memory-constrained devices with large ZO segments. 
% In contrast, non-uniform participation probabilities $p$ empower the server to prioritize clients with larger FO segments, actively driving the sampling-weighted dimension $\sum_{i=1}^N p_i d_i^{\mathrm{ZO}}$ well below the arithmetic mean. 
% While prioritizing low-variance updates incurs a persistent representation penalty $\mathcal{O}(\zeta^2 N \|p - u\|^2)$ proportional to client data heterogeneity $\zeta^2$, properly calibrating $p$ unlocks a superior optimization trajectory. 
% This principled Pareto trade-off between directional variance suppression and global objective fidelity directly inspires the algorithmic design of our dimension-aware sampling policy in Section~\ref{sec:sampling}.
\section{Dimension-Aware Client Sampling}
\label{sec:sampling}

Theorem~\ref{thm:convergence} exposes an optimization opportunity: non-uniform selection $p$ can lower the sampling-weighted dimension by prioritizing FO-dominant updates, at the cost of a representation penalty $\mathcal{O}(\zeta^2 N \|p-1/N\|^2)$. 
This principled trade-off between directional variance reduction and data representation directly drives the algorithmic design of our dimension-aware sampling policy.

\paragraph{Problem formulation}To design a principled sampling policy, we isolate the sampling-dependent terms in our multi-step trajectory descent bound (Proposition~\ref{prop:trajectory-descent}). 
Suppressing the round index, let $U_i \coloneqq (w - w_i^E)/h$ denote client $i$'s normalized local update with effective step size $h \coloneqq \eta E$, $a_i \coloneqq \mathbb{E}[U_i]$ be its expected trajectory, $v_i \coloneqq \mathbb{E}\|U_i - a_i\|^2$ be its local variance, $\bar a \coloneqq \frac{1}{N}\sum_i a_i$ be the unweighted cohort mean, and $g \coloneqq \nabla f(w)$ represent the true global gradient. 
Over the constrained simplex $\mathcal C \coloneqq \{p \in \mathbb{R}^N : \sum_i p_i = 1, \; \epsilon/N \le p_i \le 1/K\}$, minimizing the descent bound yields the ideal statistical quadratic program (QP):
\begin{equation}
  \min_{p \in \mathcal C} \;
  \left\|\sum_{i=1}^N p_i a_i - g\right\|^2
  + \frac{Lh}{K} \sum_{i=1}^N p_i \left[ v_i + 2\|a_i - \bar a\|^2 \right].
  \label{eq:sampling-qp}
\end{equation}
Problem~\eqref{eq:sampling-qp} serves as an optimal reference: the first term minimizes representation bias relative to $g$, while the second penalizes cohort variance induced by local noise $v_i$ and client drift dispersion $\|a_i - \bar a\|^2$. 
However, solving~\eqref{eq:sampling-qp} online is intractable because trajectory moments $\{a_i, v_i\}$ and gradient $g$ are unobservable prior to local execution.

\paragraph{Scalar relaxation and water-filling structure.}
To decouple directional cross-terms and eliminate the unknown gradient $g$, we relax the squared bias via Cauchy--Schwarz: $\|\sum_i p_i a_i - g\|^2 \le 2 N H_a \|p - u\|^2 + 2\|\bar a - g\|^2$, where $H_a \coloneqq \frac{1}{N}\sum_i \|a_i - \bar a\|^2$ measures empirical update dispersion and $u \coloneqq [1/N, \dots, 1/N]^\top$. 
Dropping the $p$-independent offset $2\|\bar a - g\|^2$ transforms~\eqref{eq:sampling-qp} into a separable scalar program:
\begin{equation}
  \min_{p \in \mathcal C} \; A \|p - u\|^2 + c^\top p, \quad
  \text{where} \quad A \coloneqq 2 N H_a, \quad
  c_i \coloneqq \frac{Lh}{K}\big[ v_i + 2\|a_i - \bar a\|^2 \big].
  \label{eq:scalar-qp}
\end{equation}
When $A > 0$, the KKT conditions admit a closed-form \textit{water-filling} solution:
\begin{equation}
  p_i^\star = \operatorname{clip}_{[\epsilon/N, \, 1/K]} \left( \frac{1}{N} - \frac{c_i + \lambda}{2A} \right),
  \label{eq:clip-solution_app}
\end{equation}
where scalar dual multiplier $\lambda$ enforces $\sum_i p_i^\star = 1$. 
Appendix~\ref{app:qp} provides the full relaxation and KKT derivations.
Equation~\eqref{eq:clip-solution_app} reveals the qualitative behavior of optimal sampling: it prioritizes clients with smaller noise costs $c_i$ up to the capacity cap $1/K$, while pulling coordinates toward uniform $1/N$ to bound representation error.

\paragraph{The practical HO-FL policy.}
Evaluating runtime costs $c_i$ remains costly as it requires dynamic variance tracking. 
To achieve zero runtime overhead while retaining the water-filling principle, we substitute $c_i$ with the static \textit{dimension score} $s_i \coloneqq d_i^{\mathrm{ZO}}/q$, grounded in the $(d_i^{\mathrm{ZO}} + 1)/q$ variance penalty in Lemma~\ref{lem:ho-moments}. 
Under this proxy, minimizing noise alone ($\min_{p \in \mathcal C} s^\top p$) is greedily solved by saturating the $K$ clients with the smallest scores at $1/K$, yielding $r_i \coloneqq \frac{1}{K}\mathbf{1}\{i \in R\}$, where $R$ denotes the greedy subset. 
Directly parameterizing the trade-off line segment connecting the representation-optimal distribution $u$ and the noise-optimal distribution $r$ produces our closed-form HO-FL policy:
\begin{equation}
  p_i = (1 - \beta)\frac{1}{N} + \beta r_i, \qquad \beta \in [0, 1).
  \label{eq:practical-policy}
\end{equation}
Equation~\eqref{eq:practical-policy} automatically satisfies the bounded simplex constraints with $p_i \le 1/K$ and coverage floor $p_i \ge (1 - \beta)/N$. 
Here, hyperparameter $\beta$ serves as an explicit control knob governing the theoretical trade-off in Theorem~\ref{thm:convergence}: setting $\beta = 0$ recovers uniform sampling and eliminates representation bias ($\|p - u\|^2 = 0$), whereas increasing $\beta$ tilts selection toward FO-dominant clients, actively suppressing directional variance while analytically bounding representation bias by $N\|p - u\|^2 = \beta^2(N/K - 1)$. Practical guidance for selecting a fixed $\beta$ provided in Appendix~\ref{app:beta-practice}.
\section{Experiments}
\label{sec:experiments}
In this section, we empirically evaluate \textbf{HO-FL} across diverse foundation models and downstream tasks, centered on three core research questions:
\begin{itemize}[leftmargin=18pt]
  \item \textbf{RQ1 (Task Performance):} Can HO-FL close the accuracy gap to full first-order federated optimization while accommodating heterogeneous client order boundary?
  \item \textbf{RQ2 (Memory Feasibility):} How effectively does hybrid-order local training reduce client peak activation memory, and what are the system-wide resource footprints?
  \item \textbf{RQ3 (Heterogeneity \& Trade-offs):} How does the dimension-aware sampling policy ($\beta$) interact with data non-IIDness ($\alpha$), and does it validate our theoretical convergence bound?
\end{itemize}

\subsection{Experimental Setup}
\label{sec:exp-setup}

\begin{table}[!htbp]
\centering
\begin{minipage}[c]{0.56\textwidth}
\paragraph{Federated environment and baselines.}
We simulate an edge system of $N=30$ clients with cohort size $K=6$ over 160 communication rounds.
Each active client executes $E=5$ local AdamW steps.
To emulate memory heterogeneity, clients are evenly partitioned across three label-independent boundary tiers (Table~\ref{tab:hardware-tiers}).
We benchmark HO-FL ($\beta=0.75$) against full-FO (\textbf{FedAvg}~\citep{fedavg}, \textbf{FedProx}~\citep{fedprox}), pure-ZO \textbf{DeComFL}~\citep{decomfl}, and the uniform sampling version of HO-FL ($\beta=0$).
Implementation details are in Appendix~\ref{app:experiments}.
\end{minipage}\hfill
\begin{minipage}[c]{0.41\textwidth}
\centering
\caption{\textbf{Boundary tiers.} Order Boundary setting across 3 tiers.}
\label{tab:hardware-tiers}
\small
\begin{tabular}{lcccc}
\toprule
\multirow{2}{*}{Model} & \multirow{2}{*}{Layers} & \multicolumn{3}{c}{Order Boundary} \\
\cmidrule(lr){3-5}
 & & T1 & T2 & T3 \\
\midrule
OPT-125M     & 12 & 3 & 6  & 9  \\
Qwen2.5-1.5B & 28 & 8 & 16 & 24 \\
SmolLM3-3B   & 36 & 9 & 18 & 27 \\
\bottomrule
\end{tabular}
\end{minipage}
\end{table}

\paragraph{Models and benchmarks.}
We evaluate HO-FL by fine-tuning OPT-125M, Qwen2.5-1.5B, and SmolLM3-3B~\citep{opt,qwen25,smollm3} with LoRA adapters~\citep{lora} .
The benchmarks include SST-2, BoolQ, SciQ, SQuAD 1.1, and AG News~\citep{sst2,boolq,sciq,squad,zhang2015character}.
Tasks are evaluated under IID client partitions, while AG News additionally incorporates Dirichlet label-skewed partitions $\operatorname{Dir}(\alpha)$ with $\alpha \in \{1, 0.5, 0.1\}$ to evaluate non-IID robustness.

\subsection{Empirical Evaluation and Analysis}
\label{sec:exp-results}

\begin{table}[!htbp]
\centering
\caption{\textbf{Language-task accuracy (\%).} Reported as $\text{mean}_{(\text{SD})}$ across three seeds. All tasks are IID except AG News ($\alpha=0.5$).}
\label{tab:llm-v2}
\small
\begin{tabular}{llccccc}
\toprule
Model & Task & FedAvg & FedProx & DeComFL & \shortstack{HO-FL\\($\beta=0$)} & \shortstack{HO-FL\\($\beta=0.75$)} \\
\midrule
OPT-125M & SST-2 & $87.04_{(0.11)}$ & $87.12_{(0.13)}$ & $51.95_{(0.00)}$ & $86.01_{(0.11)}$ & $86.54_{(0.24)}$ \\
 & BoolQ & $59.84_{(0.37)}$ & $59.85_{(0.46)}$ & $55.87_{(0.05)}$ & $57.91_{(0.18)}$ & $59.44_{(0.96)}$ \\
 & SciQ & $25.50_{(1.30)}$ & $25.43_{(1.04)}$ & $25.53_{(0.85)}$ & $24.97_{(0.80)}$ & $26.10_{(1.64)}$ \\
 & AG News ($\alpha=0.5$) & $78.24_{(2.97)}$ & $78.29_{(2.99)}$ & $23.38_{(0.08)}$ & $59.58_{(2.24)}$ & $62.85_{(6.52)}$ \\
\midrule
Qwen2.5-1.5B & SST-2 & $93.27_{(0.26)}$ & $93.43_{(0.26)}$ & $71.25_{(0.18)}$ & $91.09_{(0.18)}$ & $92.81_{(0.07)}$ \\
 & BoolQ & $79.06_{(0.19)}$ & $79.11_{(0.08)}$ & $72.44_{(0.12)}$ & $77.00_{(0.40)}$ & $78.72_{(0.24)}$ \\
 & SciQ & $90.83_{(0.32)}$ & $90.87_{(0.31)}$ & $88.87_{(0.46)}$ & $90.17_{(0.25)}$ & $90.47_{(0.29)}$ \\
 & AG News ($\alpha=0.5$) & $88.15_{(0.21)}$ & $88.23_{(0.17)}$ & $55.83_{(0.08)}$ & $85.51_{(0.43)}$ & $86.69_{(0.56)}$ \\
\midrule
SmolLM3-3B  & SST-2 & $94.04_{(0.11)}$ & $94.07_{(0.07)}$ & $63.38_{(0.18)}$ & $93.16_{(0.07)}$ & $94.04_{(0.30)}$ \\
 & BoolQ & $86.15_{(0.05)}$ & $86.01_{(0.14)}$ & $82.59_{(0.07)}$ & $84.32_{(0.12)}$ & $85.37_{(0.43)}$ \\
 & SciQ & $90.50_{(0.17)}$ & $90.57_{(0.35)}$ & $89.13_{(0.38)}$ & $89.57_{(0.35)}$ & $89.77_{(0.21)}$ \\
 & AG News ($\alpha=0.5$) & $89.38_{(0.14)}$ & $89.37_{(0.18)}$ & $74.70_{(0.05)}$ & $87.11_{(0.22)}$ & $87.52_{(0.55)}$ \\
\bottomrule
\end{tabular}
\end{table}

\paragraph{Downstream task performance (RQ1).}
As shown in Table~\ref{tab:llm-v2}, HO-FL with $\beta=0.75$ closely approaches FO baselines accuracy across language understanding tasks, remaining within 0.5\%--0.8\% of FedAvg on most benchmarks while surpassing DeComFL by up to 34.59\% (86.54\% vs. 51.95\% on OPT-125M/SST-2).
Moreover, dimension-aware sampling consistently outperforms uniform hybrid selection with $\beta=0$, yielding up to a 3.27\% accuracy gain under label skew on AG News at $\alpha=0.5$.

On generative QA evaluated on SQuAD 1.1 (Table~\ref{tab:squad-v2}), where pure ZO degrades severely as DeComFL achieves only 8.54\% F1 on OPT-125M, HO-FL reliably restores convergence by lifting to 52.43\%.
Specifically, dimension-aware sampling improves F1 over uniform hybrid selection by 1.3\%--5.2\% across architectures, narrowing the performance gap with FedAvg to within 0.7\%--2.3\%.

\begin{table}[!htbp]
\centering
\caption{\textbf{SQuAD 1.1 performance (\%).} Full-validation EM/F1, reported as $\text{mean}_{(\text{SD})}$ across three seeds.}
\label{tab:squad-v2}
\small
\begin{tabular}{llccccc}
\toprule
Model & Metric & FedAvg & FedProx & DeComFL & \shortstack{HO-FL\\($\beta=0$)} & \shortstack{HO-FL\\($\beta=0.75$)} \\
\midrule
OPT-125M & EM & $43.95_{(0.04)}$ & $43.92_{(0.10)}$ & $1.29_{(0.05)}$ & $37.16_{(0.26)}$ & $41.87_{(0.32)}$ \\
 & F1 & $54.68_{(0.02)}$ & $54.65_{(0.08)}$ & $8.54_{(0.09)}$ & $47.24_{(0.21)}$ & $52.43_{(0.20)}$ \\
\midrule
Qwen2.5-1.5B & EM & $79.65_{(0.14)}$ & $79.70_{(0.08)}$ & $61.70_{(0.11)}$ & $76.67_{(0.05)}$ & $78.71_{(0.19)}$ \\
 & F1 & $87.59_{(0.07)}$ & $87.61_{(0.04)}$ & $73.14_{(0.16)}$ & $85.34_{(0.03)}$ & $86.92_{(0.10)}$ \\
\midrule
SmolLM3-3B  & EM & $81.52_{(0.10)}$ & $81.47_{(0.20)}$ & $47.98_{(0.09)}$ & $77.80_{(0.11)}$ & $79.54_{(0.20)}$ \\
 & F1 & $89.93_{(0.10)}$ & $89.90_{(0.16)}$ & $68.43_{(0.04)}$ & $87.25_{(0.04)}$ & $88.57_{(0.13)}$ \\
\bottomrule
\end{tabular}
\end{table}

\FloatBarrier
\begin{figure}[!htbp]
\centering
\includegraphics[width=\linewidth]{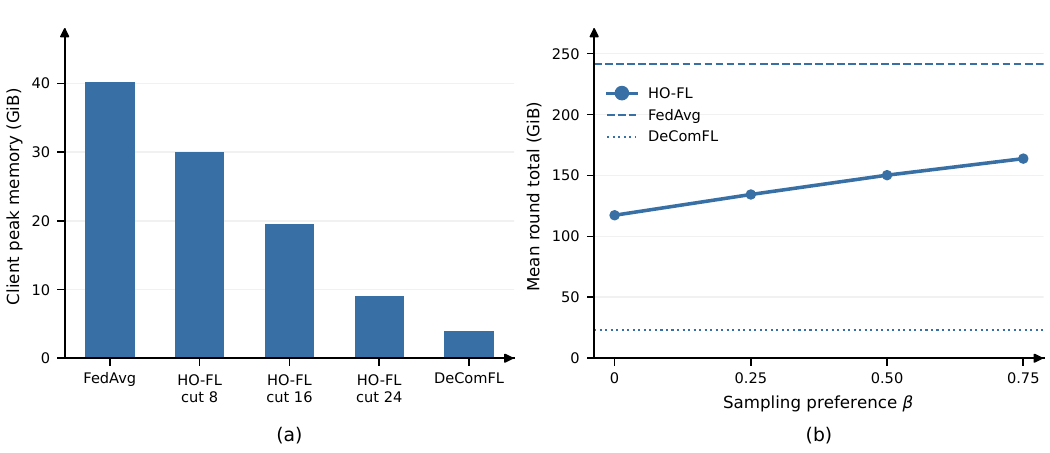}
\caption{\textbf{Single-client peak and system average memory footprints on Qwen2.5-1.5B.} Left: measured single-client peak allocated memory across different order boundaries during local training (batch size 16 and length 1024). Right: round-averaged cohort memory demand $\overline C(L)$ across participating clients ($N=30, K=6$) under varying $\beta$.}
\label{fig:memory-v2}
\end{figure}

\paragraph{On-device memory footprint and feasibility (RQ2).}
As illustrated in Figure~\ref{fig:memory-v2}, HO-FL systematically resolves the activation memory bottleneck during local training.
At sequence length 1024, full BP (FedAvg) incurs a prohibitive peak memory of 40.26~GiB on Qwen2.5-1.5B.
By setting the boundary at block 24 to truncate the backward computation graph, HO-FL slashes peak allocated memory to 9.08~GiB—achieving a \textbf{77.4\% reduction} that brings 1.5B LLM adaptation within reach of commodity 12~GiB devices.

To evaluate distributed system overhead, we profile the expected cohort memory demand $\overline C(L) = T^{-1}\sum_t\sum_{i\in S_t}m_i(L)$ at sequence length 1024 (Figure~\ref{fig:memory-v2}).
Uniform hybrid sampling ($\beta=0$) and dimension-aware sampling ($\beta=0.75$) require an average of 117.32 GiB and 163.78 GiB per round, respectively, yielding substantial savings over homogeneous FedAvg (241.57 GiB).
Although a higher $\beta$ moderately increases aggregate resource utilization by prioritizing FO-dominant clients, it accelerates global convergence while strictly respecting each client's physical memory budget.

\subsection{Sampling Preference and Data Heterogeneity (RQ3)}
\label{sec:heterogeneity-study}
\begin{table}[tbh]
\centering
\caption{\textbf{Sampling preference and data heterogeneity.} Qwen/AG News final test accuracy (\%), reported as $\text{mean}_{(\text{SD})}$ across three seeds.}
\label{tab:beta-v2}
\begin{tabular}{lcccc}
\toprule
Distribution & $\beta=0$ & $\beta=0.25$ & $\beta=0.5$ & $\beta=0.75$ \\
\midrule
IID & $86.75_{(0.13)}$ & $87.33_{(0.03)}$ & $87.58_{(0.08)}$ & $87.91_{(0.03)}$ \\
$\alpha=1$ & $86.22_{(0.47)}$ & $86.64_{(0.23)}$ & $87.11_{(0.15)}$ & $87.26_{(0.29)}$ \\
$\alpha=0.5$ & $85.51_{(0.43)}$ & $86.31_{(0.48)}$ & $86.58_{(0.44)}$ & $86.69_{(0.56)}$ \\
$\alpha=0.1$ & $83.39_{(1.46)}$ & $83.83_{(0.92)}$ & $82.53_{(1.81)}$ & $82.66_{(3.07)}$ \\
\bottomrule
\end{tabular}
\end{table}

On Qwen/AG News (Table~\ref{tab:beta-v2}), $\beta=0.75$ improves over HO-FL ($\beta=0$) by 1.17 points under IID, 1.04 points at $\alpha=1$, and 1.18 points at $\alpha=0.5$, with gains in every seed.
Under stronger skew ($\alpha=0.1$), $0.25$ is preferable: $0.5$ and $0.75$ fall below HO-FL ($\beta=0$) by 0.87 and 0.74 points, and $0.75$ has SD 3.07.
This supports the trade-off between update quality and data representation as shown in Theorem~\ref{thm:convergence}.
Appendix~\ref{app:ablations-v2} further reports $E/q/\mu$ ablations on OPT/SST-2.

\section{Conclusion}
\label{sec:conclusion}

In this paper, we presented HO-FL, the first fully federated hybrid-order learning framework designed for heterogeneous edge devices. By executing ZO optimization on the bottom segment and FO backpropagation on the top segment entirely on each client, HO-FL enables memory-constrained devices to adapt their local order boundaries to memory budgets. Our theoretical analysis characterizes the convergence behavior of HO-FL and introduces a principled dimension-aware client sampling policy that balances directional gradient variance against global data representation. Extensive empirical evaluations across diverse tasks demonstrate that HO-FL achieves performance competitive with full first-order federated learning while slashing client peak memory by up to 77.4\%. 
Promising avenues for future work include extending HO-FL to support runtime-adaptive order boundaries under fluctuating edge resources and developing dynamic schedules for $\beta$ .
\label{page:main-end}
\clearpage
\bibliography{references}
\bibliographystyle{iclr2027_conference}
\clearpage
\appendix
% List only appendix sections and subsections, using their actual page numbers.
\startcontents[appendices]
\section*{Appendix Contents}
\printcontents[appendices]{}{1}{\setcounter{tocdepth}{2}}
\clearpage

% ===== BEGIN appendices/unified_framework.tex =====
\section{A Unified Formulation of Hybrid-Order Optimization}
\label{app:unified-framework}

This section formalizes the generalized hybrid-order (HO) optimization framework referenced in Section~\ref{sec:related}, unifying existing distributed hybrid-order paradigms under a common structural formulation.

\subsection{Structural Decomposition and Forward Propagation}
Consider a deep neural network partitioned across a specified cut layer into a bottom sub-network $f_{\mathrm{ZO}}(\cdot; w^{\mathrm{ZO}})$ and a top sub-network $f_{\mathrm{FO}}(\cdot; w^{\mathrm{FO}})$. The parameter space is canonically decoupled as:
\begin{equation}
  \begin{aligned}
    w &= \left[ (w^{\mathrm{ZO}})^\top, (w^{\mathrm{FO}})^\top \right]^\top \in \mathbb{R}^d,\qquad
    d^{\mathrm{ZO}} + d^{\mathrm{FO}} = d,\\
    w^{\mathrm{ZO}} &\in \mathbb{R}^{d^{\mathrm{ZO}}},\qquad
    w^{\mathrm{FO}} \in \mathbb{R}^{d^{\mathrm{FO}}}.
  \end{aligned}
  \label{eq:app-param-split}
\end{equation}
Given an input sample $x$ with target label $y$, forward execution proceeds sequentially through the boundary activation $z$ to produce the task prediction $\hat{y}$:
\begin{equation}
  \begin{aligned}
    z &= f_{\mathrm{ZO}}(x; w^{\mathrm{ZO}}),\qquad
    \hat{y} = f_{\mathrm{FO}}(z; w^{\mathrm{FO}}),\\
    \ell(w; x, y) &= \ell(\hat{y}, y).
  \end{aligned}
  \label{eq:app-forward}
\end{equation}
In all hybrid architectures, the top segment $w^{\mathrm{FO}}$ is optimized via standard first-order backpropagation, yielding the exact gradient $g^{\mathrm{FO}} = \nabla_{w^{\mathrm{FO}}} \ell(\hat{y}, y)$. The bottom segment $w^{\mathrm{ZO}}$ is updated via zeroth-order perturbations evaluated over a local surrogate objective $\mathcal{L}_{\mathrm{local}}(v)$ at candidate parameters $v$.

\subsection{Taxonomy of Local Surrogate Objectives}
The fundamental distinction among existing distributed hybrid-order methods lies in how the surrogate objective $\mathcal{L}_{\mathrm{local}}(v)$ is constructed:
\begin{itemize}[leftmargin=18pt]
  \item \textbf{Exact End-to-End Evaluation (HOSL~\citep{hosl}):} HOSL perturbs the bottom segment while evaluating the entire forward graph:
  \begin{equation}
    \mathcal{L}_{\mathrm{local}}^{\mathrm{HOSL}}(v) = \ell\left(f_{\mathrm{FO}}(f_{\mathrm{ZO}}(x; v); w^{\mathrm{FO}}), y\right).
    \label{eq:app-hosl}
  \end{equation}
  Each perturbation query requires transmitting the perturbed activation $f_{\mathrm{ZO}}(x; v)$ to the server and returning the scalar loss, incurring substantial per-iteration communication.
  
  \item \textbf{First-Order Activation Feedback (HO-SFL~\citep{hosfl}):} HO-SFL freezes the cut-layer activation gradient $\delta = \nabla_z \ell(f_{\mathrm{FO}}(z; w^{\mathrm{FO}}), y)$ evaluated at the unperturbed anchor, forming a linear inner-product surrogate:
  \begin{equation}
    \mathcal{L}_{\mathrm{local}}^{\mathrm{HO\text{-}SFL}}(v) = \langle \delta, f_{\mathrm{ZO}}(x; v) \rangle.
    \label{eq:app-hosfl}
  \end{equation}
  This formulation strictly confines all perturbation forward passes to the bottom sub-network, eliminating activation communication during the ZO search phase.
  
  \item \textbf{Auxiliary Network Decoupling (HERON-SFL~\citep{heron}):} HERON-SFL attaches a lightweight client-side auxiliary classifier $h(\cdot; \theta)$ to approximate the top segment's guidance locally:
  \begin{equation}
    \mathcal{L}_{\mathrm{local}}^{\mathrm{HERON}}(v, \theta) = \ell\left(h(f_{\mathrm{ZO}}(x; v); \theta), y\right).
    \label{eq:app-heron}
  \end{equation}
  Both $v$ and the auxiliary parameters $\theta$ are updated locally via ZO perturbations to minimize cross-node dependency, at the cost of auxiliary architectural modeling bias.
\end{itemize}

\clearpage
\subsection{Architectural Comparison: Split Computing vs. HO-FL}
As summarized in Table~\ref{tab:unified-comparison}, HOSL, HO-SFL, and HERON-SFL use Split Learning (SL) or Split Federated Learning (SFL) workflows, with cross-node activation transmissions between clients and the server.

In contrast, \textbf{HO-FL} adopts the computational efficiency of the linear surrogate~\eqref{eq:app-hosfl} but breaks away from the split paradigm: it executes the complete hybrid-order pipeline on-device, natively accommodates client-adaptive boundaries $b_i$ matched to heterogeneous physical memory constraints, and aggregates model updates after $E$ local training steps using seed replay for the ZO segment and transmitted parameter differences for the FO segment.

\begin{table}[htbp]
\centering
\caption{\textbf{Structural comparison of distributed hybrid-order paradigms.}}
\label{tab:unified-comparison}
\small
\begin{tabular}{@{}
  >{\raggedright\arraybackslash}p{\dimexpr0.185\linewidth-1.48\tabcolsep\relax}
  >{\raggedright\arraybackslash}p{\dimexpr0.16\linewidth-1.28\tabcolsep\relax}
  >{\raggedright\arraybackslash}p{\dimexpr0.225\linewidth-1.80\tabcolsep\relax}
  >{\raggedright\arraybackslash}p{\dimexpr0.27\linewidth-2.16\tabcolsep\relax}
  >{\raggedright\arraybackslash}p{\dimexpr0.16\linewidth-1.28\tabcolsep\relax}
  @{}}
\toprule
Method & System Paradigm & Boundary Flexibility & Activation Uploads & Local Steps per Aggregation \\
\midrule
\mbox{HOSL}\newline\citep{hosl} & Split Learning & Fixed (Single Client) & $2q+1$ uploads / step & N/A (Single Client) \\
\mbox{HO-SFL}\newline\citep{hosfl} & Split FedLearning & Identical & 1 upload / step & 1 \\
\mbox{HERON-SFL}\newline\citep{heron} & Split FedLearning & Identical  & Every $k$ local steps & $h \ge 1$ \\
\midrule
\textbf{HO-FL (Ours)} & \textbf{Federated Learning} & \textbf{Heterogeneous (Adaptive)} & \textbf{None (Fully On-Device)} & $\boldsymbol{E \ge 1}$ \\
\bottomrule
\end{tabular}
\end{table}
% ===== END appendices/unified_framework.tex =====

\section{Algorithm and Communication Details}
\label{app:algorithms}

\subsection{Fixed-size dependent rounding}
\label{app:depround}
Given normalized inclusion probabilities $p_i\in[0,1/K]$ with
$\sum_i p_i=1$, Algorithm~\ref{alg:depround} rounds $x=Kp$ to a binary vector.
It preserves $\mathbb E[X_i]=x_i$ and $\sum_iX_i=K$ exactly and satisfies
$\mathbb E[X_iX_j]\le x_ix_j$ for $i\ne j$.
One randomized permutation is drawn before the pairwise procedure; the next
two fractional entries are taken in this order.
When all inputs are equal, permutation symmetry makes the output uniform
over all $K$-subsets.

% Shared algorithm styling. Palette follows the current Overleaf Figure 1.
% Load after algorithm, algpseudocode, xcolor, and hyperref.
\ifdefined\HOFLAlgorithmStyleLoaded\else
\def\HOFLAlgorithmStyleLoaded{1}
\definecolor{hoAlgZO}{HTML}{244A78}
\definecolor{hoAlgFO}{HTML}{D98732}
\definecolor{hoAlgServer}{HTML}{39816B}
\definecolor{hoAlgGray}{HTML}{65717D}
\newcommand{\HOFunc}[2]{\textcolor{#1}{\textbf{\textsc{#2}}}}
\newcommand{\HOCall}[4]{\hyperref[#2]{\HOFunc{#1}{#3}}\ensuremath{(#4)}}
\newcommand{\HOComment}[1]{\Comment{\textcolor{hoAlgGray}{\scriptsize #1}}}
\newcommand{\HOPhase}[1]{\Statex\textcolor{hoAlgGray}{\textit{// #1}}}
\fi

\begin{algorithm}[!htbp]
\caption{Function: DepRound}
\label{alg:depround}
\begin{algorithmic}[1]
\Require $x\in[0,1]^N$, $\sum_i x_i=K\in\mathbb N$.
\Function{\HOFunc{hoAlgServer}{DepRound}}{$x$}
  \State $\pi\gets$ a uniformly random permutation of $\{1,\ldots,N\}$.
  \While{$|\{i:0<x_i<1\}|\ge2$}
    \State $i,j\gets$ the first two fractional indices in order $\pi$.
    \State $\delta_+\gets\min(1-x_i,x_j)$; $\delta_-\gets\min(x_i,1-x_j)$.
    \State $a\sim\operatorname{Uniform}(0,1)$.
    \If{$a<\delta_-/(\delta_++\delta_-)$}
      \State $(x_i,x_j)\gets(x_i+\delta_+,x_j-\delta_+)$.
    \Else
      \State $(x_i,x_j)\gets(x_i-\delta_-,x_j+\delta_-)$.
    \EndIf
  \EndWhile
  \State \Return \textcolor{hoAlgServer}{$S=\{i:x_i=1\}$}.
    \HOComment{Used by Algorithm~\ref{alg:hofl}}
\EndFunction
\end{algorithmic}
\end{algorithm}

Every step preserves the sum, preserves each coordinate in expectation, and
makes at least one coordinate integral.
For a pair updated together, its product has conditional expected decrement
$\delta_+\delta_-$; when only one coordinate of a fixed pair is updated,
that pair's product is a martingale.
Thus pair products are supermartingales, giving the stated nonpositive
inclusion covariances.
Because the sum is an integer, one fractional coordinate cannot remain.
Implementations snap values within floating-point tolerance of $0$ or $1$
and check the final cardinality.

\paragraph{The probability rule used in the main experiments.}
Let $R$ contain the $K$ clients with the smallest
$d_i^{\mathrm{ZO}}/q$, with ties resolved by a seeded random permutation drawn once at run start
and retained for stable sorting throughout the run.
For $u_i=1/N$,
\begin{equation}
  r_i=\frac{\mathbf1\{i\in R\}}{K},
  \qquad p_i=(1-\beta)u_i+\beta r_i,
  \qquad S_t=\mathrm{DepRound}(Kp).
  \label{eq:practical-depround}
\end{equation}
If all resource scores are equal, set $r=u$.
The IID main comparisons use $\beta=0.75$; the sampling sensitivity study uses $\beta\in\{0,0.25,0.5,0.75\}$, and the OPT/SST-2 local-training ablations fix $\beta=0.5$.
The resource ranking and $p$ are fixed before training and held constant across rounds. Dependent rounding uses a fresh
random permutation each round to draw the participants.
The resulting $p$ always satisfies $p_i\le1/K$ and
$p_i\ge(1-\beta)/N$.

\subsection{Local updates and server reconstruction}
\label{app:local-step}
\label{app:reconstruct}
Algorithm~\ref{alg:client-update} implements the $E$-step local update in
\eqref{eq:local-updates} and returns the directional scalars and top-segment
parameter difference. Transient quantities omit round and step indices and
are recomputed at each step; $\ell(\hat y_i,y)$ denotes the prediction-space
form of $\ell(w_i;\xi)$. Direction streams are independent across clients
and steps and separate from minibatch and random-layer streams.
For full-FO clients, the ZO loop and scalar record are empty.

% Shared algorithm styling. Palette follows the current Overleaf Figure 1.
% Load after algorithm, algpseudocode, xcolor, and hyperref.
\ifdefined\HOFLAlgorithmStyleLoaded\else
\def\HOFLAlgorithmStyleLoaded{1}
\definecolor{hoAlgZO}{HTML}{244A78}
\definecolor{hoAlgFO}{HTML}{D98732}
\definecolor{hoAlgServer}{HTML}{39816B}
\definecolor{hoAlgGray}{HTML}{65717D}
\newcommand{\HOFunc}[2]{\textcolor{#1}{\textbf{\textsc{#2}}}}
\newcommand{\HOCall}[4]{\hyperref[#2]{\HOFunc{#1}{#3}}\ensuremath{(#4)}}
\newcommand{\HOComment}[1]{\Comment{\textcolor{hoAlgGray}{\scriptsize #1}}}
\newcommand{\HOPhase}[1]{\Statex\textcolor{hoAlgGray}{\textit{// #1}}}
\fi

\begin{algorithm}[!htbp]
\caption{Function: ClientUpdate}
\label{alg:client-update}
\label{alg:local-ho}
\begin{algorithmic}[1]
\Require $b_i,E,q,\mu,\eta$ and dispatched $\{\mathrm{seed}_{i,t,j}^{e}\}_{e,j}$ from Algorithm~\ref{alg:hofl}.
\Function{\HOFunc{hoAlgZO}{ClientUpdate}}{$i,t,w_t$}
  \State $w_{i,t}^{0}\gets w_t$.
  \For{$e=0,\ldots,E-1$}
    \State $\xi_{i,t}^{e}=(x,y)\sim\mathcal D_i$.
    \HOPhase{Forward and top-segment BP}
    \State $z_i\gets f_{i,\mathrm{ZO}}(x;w_{i,t}^{\mathrm{ZO},e})$.
      \HOComment{Forward without caching activation}
    \State $\hat y_i\gets f_{i,\mathrm{FO}}(z_i;w_{i,t}^{\mathrm{FO},e})$.
    \State \textcolor{hoAlgFO}{$g_i^{\mathrm{FO}}\gets\nabla_{w_{i,t}^{\mathrm{FO},e}}\ell(\hat y_i,y)$},
      \textcolor{hoAlgFO}{$\delta_i\gets\nabla_{z_i}\ell(\hat y_i,y)$}.
      \HOComment{Eq.~\eqref{eq:top-feedback}}
    \HOPhase{Surrogate and directional scalars}
    \State $\mathcal L_i(v)\gets\langle\delta_i,f_{i,\mathrm{ZO}}(x;v)\rangle$.
      \HOComment{Eq.~\eqref{eq:surrogate}}
    \State $\widehat g_i^{\mathrm{ZO}}\gets0$.
    \If{$d_i^{\mathrm{ZO}}>0$}
      \For{$j=1,\ldots,q$}
        \State \textcolor{hoAlgZO}{$u_{i,j}\gets\operatorname{RNG}(\mathrm{seed}_{i,t,j}^{e})$}.
          \HOComment{Eq.~\eqref{eq:seed-direction}}
        \State \textcolor{hoAlgZO}{$s_{i,t,j}^{e}\gets
          \frac{\mathcal L_i(w_{i,t}^{\mathrm{ZO},e}+\mu u_{i,j})-\mathcal L_i(w_{i,t}^{\mathrm{ZO},e})}{\mu}$}.
          \HOComment{Eq.~\eqref{eq:local-zo}}
        \State \textcolor{hoAlgZO}{$\widehat g_i^{\mathrm{ZO}}\gets
          \widehat g_i^{\mathrm{ZO}}+\frac1q s_{i,t,j}^{e}u_{i,j}$}.
      \EndFor
    \EndIf
    \HOPhase{Update both segments at the same anchor}
    \State \textcolor{hoAlgZO}{$w_{i,t}^{\mathrm{ZO},e+1}\gets w_{i,t}^{\mathrm{ZO},e}-\eta\widehat g_i^{\mathrm{ZO}}$}.
      \HOComment{Eq.~\eqref{eq:local-updates}}
    \State \textcolor{hoAlgFO}{$w_{i,t}^{\mathrm{FO},e+1}\gets w_{i,t}^{\mathrm{FO},e}-\eta g_i^{\mathrm{FO}}$}.
  \EndFor
  \State \textcolor{hoAlgFO}{$\Delta_{i,t}^{\mathrm{FO}}\gets
    w_{i,t}^{\mathrm{FO},E}-w_{i,t}^{\mathrm{FO},0}$}.
    \HOComment{Eq.~\eqref{eq:top-difference}}
  \State \Return $(\{s_{i,t,j}^{e}\}_{e,j},\Delta_{i,t}^{\mathrm{FO}})$.
\EndFunction
\end{algorithmic}
\end{algorithm}

Algorithm~\ref{alg:reconstruct} regenerates the directions from the shared
seeds and recovers the bottom update using~\eqref{eq:bottom-reconstruction}, allowing
Algorithm~\ref{alg:hofl} to average updates across heterogeneous boundaries
via~\eqref{eq:aggregation-step}. 

% Shared algorithm styling. Palette follows the current Overleaf Figure 1.
% Load after algorithm, algpseudocode, xcolor, and hyperref.
\ifdefined\HOFLAlgorithmStyleLoaded\else
\def\HOFLAlgorithmStyleLoaded{1}
\definecolor{hoAlgZO}{HTML}{244A78}
\definecolor{hoAlgFO}{HTML}{D98732}
\definecolor{hoAlgServer}{HTML}{39816B}
\definecolor{hoAlgGray}{HTML}{65717D}
\newcommand{\HOFunc}[2]{\textcolor{#1}{\textbf{\textsc{#2}}}}
\newcommand{\HOCall}[4]{\hyperref[#2]{\HOFunc{#1}{#3}}\ensuremath{(#4)}}
\newcommand{\HOComment}[1]{\Comment{\textcolor{hoAlgGray}{\scriptsize #1}}}
\newcommand{\HOPhase}[1]{\Statex\textcolor{hoAlgGray}{\textit{// #1}}}
\fi

\begin{algorithm}[!htbp]
\caption{Function: Reconstruct}
\label{alg:reconstruct}
\begin{algorithmic}[1]
\Require $b_i,E,q,\eta$ and known $\{\mathrm{seed}_{i,t,j}^{e}\}_{e,j}$ from Algorithm~\ref{alg:hofl}.
\Function{\HOFunc{hoAlgServer}{Reconstruct}}{$i,t,\{s_{i,t,j}^{e}\}_{e,j},\Delta_{i,t}^{\mathrm{FO}}$}
  \State \textcolor{hoAlgZO}{$u_{i,t,j}^{e}\gets\operatorname{RNG}(\mathrm{seed}_{i,t,j}^{e})$},
    $e=0,\ldots,E-1$, $j=1,\ldots,q$.
  \State \textcolor{hoAlgZO}{$\Delta_{i,t}^{\mathrm{ZO}}\gets
    -\frac\eta q\sum_{e=0}^{E-1}\sum_{j=1}^{q}s_{i,t,j}^{e}u_{i,t,j}^{e}$}.
    \HOComment{Eq.~\eqref{eq:bottom-reconstruction}}
  \State $\displaystyle\Delta_{i,t}\gets
    \begin{bmatrix}\textcolor{hoAlgZO}{\Delta_{i,t}^{\mathrm{ZO}}}\\
      \textcolor{hoAlgFO}{\Delta_{i,t}^{\mathrm{FO}}}\end{bmatrix}\in\mathbb R^d$.
    \HOComment{Shared coordinates; boundary $b_i$}
  \State \Return $\Delta_{i,t}$.
\EndFunction
\end{algorithmic}
\end{algorithm}

\subsection{Seed Replay for Local AdamW}
\label{app:adamw-replay}

While our theoretical analysis in Section~\ref{sec:theory} is formulated with standard SGD, practical language model fine-tuning employs adaptive optimizers such as AdamW~\citep{adamw}. 
Under SGD, bottom-segment updates depend linearly on perturbation directions, permitting direct closed-form reconstruction via Eq.~\eqref{eq:bottom-reconstruction}. 
In contrast, AdamW applies non-linear coordinate-wise second-moment normalization and decoupled weight decay, which prevents the server from simply summing directional scalars across local iterations.

To retain dimension-free uplink communication under AdamW, HO-FL executes a \textit{deterministic seed replay} at the server. 
Crucially, in our federated setting, each selected client's local AdamW optimizer states (the first- and second-moment buffers $m$ and $v$) are reset to zero upon selection at each round. 
This stateless convention eliminates the need to transmit or maintain bulky historical optimizer states across rounds, allowing the server to reconstruct the exact client trajectory using only the dispatched seeds and uploaded directional scalars:
\begin{enumerate}[leftmargin=18pt]
  \item \textbf{Momentum Initialization:} At the start of round $t$, the server initializes the client's bottom parameters with the current global model segment, $w_{i,t}^{\mathrm{ZO},0} = w_t^{\mathrm{ZO}}$, and resets the moment buffers to zero ($m_{-1} = 0, v_{-1} = 0$).
  \item \textbf{Direction Regeneration \& Gradient Recovery:} For each local step $e \in \{0, \dots, E-1\}$, the server regenerates the exact perturbation directions $u_{i,t,j}^{e} = \operatorname{RNG}(\mathrm{seed}_{i,t,j}^{e})$ using the synchronized seed sequence, reconstructing the zeroth-order gradient estimate:
  \begin{equation}
    \widehat g_{i,t}^{\mathrm{ZO},e} = \frac{1}{q} \sum_{j=1}^{q} s_{i,t,j}^{e} u_{i,t,j}^{e}.
    \label{eq:replay-gradient}
  \end{equation}
  \item \textbf{Synchronous AdamW Step:} The server updates the local optimizer states and advances the parameter block via coordinate-wise AdamW:
  \begin{align}
    m_e &= \beta_1 m_{e-1} + (1-\beta_1)\widehat g_{i,t}^{\mathrm{ZO},e}, \qquad
    v_e = \beta_2 v_{e-1} + (1-\beta_2)(\widehat g_{i,t}^{\mathrm{ZO},e})^2, \\
    \widehat m_e &= \frac{m_e}{1-\beta_1^{e+1}}, \qquad
    \widehat v_e = \frac{v_e}{1-\beta_2^{e+1}}, \\
    w_{i,t}^{\mathrm{ZO},e+1} &= (1-\eta_e\lambda_{\mathrm{wd}})w_{i,t}^{\mathrm{ZO},e} - \eta_e\frac{\widehat m_e}{\sqrt{\widehat v_e}+\epsilon_{\mathrm{opt}}}.
    \label{eq:adamw-replay}
  \end{align}
\end{enumerate}
After executing all $E$ steps, the server obtains the final bottom parameters $w_{i,t}^{\mathrm{ZO},E}$ and derives the exact bottom update $\Delta_{i,t}^{\mathrm{ZO}} = w_{i,t}^{\mathrm{ZO},E} - w_t^{\mathrm{ZO}}$.

% ===== BEGIN appendices/theory_v3.tex =====
\section{Convergence Analysis}
\label{app:theory}

This appendix provides complete, self-contained proofs for the theoretical statements in Section~\ref{sec:theory} and Section~\ref{sec:sampling}.

\subsection{Assumptions}
\label{app:conditions}

\begin{assumption}[Objective Regularity]
\label{ass:objective}
Each local function $f_i: \mathbb{R}^d \to \mathbb{R}$ is continuously differentiable and $L$-smooth:
\begin{equation}
  \|\nabla f_i(w) - \nabla f_i(v)\| \le L \|w - v\|, \quad \forall w, v \in \mathbb{R}^d, \; \forall i \in [N].
  \label{eq:ass-smooth}
\end{equation}
Furthermore, the global objective function $f(w) \coloneqq \frac{1}{N}\sum_{i=1}^N f_i(w)$ is bounded below by $f_{\inf}$.
\end{assumption}

\begin{assumption}[Unbiased Reference Stochastic Gradients]
\label{ass:reference}
For any anchor parameter $w \in \mathbb{R}^d$ and a fresh stochastic minibatch $\xi$, conditional on the preceding local history, the reference minibatch gradient $G_i(w, \xi) \coloneqq \nabla_w \ell(w; \xi)$ satisfies:
\begin{align}
  \mathbb{E}_\xi \left[ G_i(w, \xi) \right] &= \nabla f_i(w), \label{eq:ass-unbiased}\\
  \mathbb{E}_\xi \left[ \|G_i(w, \xi) - \nabla f_i(w)\|^2 \right] &\le \sigma_i^2, \label{eq:ass-var-full}\\
  \mathbb{E}_\xi \left[ \|\Pi_i (G_i(w, \xi) - \nabla f_i(w))\|^2 \right] &\le \sigma_{i,\mathrm{ZO}}^2 \le \sigma_i^2, \label{eq:ass-var-zo}
\end{align}
where $\Pi_i \coloneqq \operatorname{diag}(I_{d_i^{\mathrm{ZO}}}, 0_{d_i^{\mathrm{FO}}}) \in \mathbb{R}^{d \times d}$ is the orthogonal projector onto client $i$'s bottom segment.
\end{assumption}

\begin{assumption}[Surrogate Objective Regularity]
\label{ass:surrogate-new}
For client $i$, let $\delta_i(w, \xi) \coloneqq \nabla_{z_i} \ell(f_{i,\mathrm{FO}}(z_i; w^{\mathrm{FO}}), y)$ denote the boundary gradient evaluated at the unperturbed model $w$, where $z_i = f_{i,\mathrm{ZO}}(x; w^{\mathrm{ZO}})$.
The local surrogate objective $\mathcal{L}_i(v; w, \xi) \coloneqq \langle \delta_i(w, \xi), f_{i,\mathrm{ZO}}(x; v) \rangle$ is $L_i^s$-smooth with respect to bottom parameters $v$:
\begin{equation}
  \|\nabla_v \mathcal{L}_i(v_1; w, \xi) - \nabla_v \mathcal{L}_i(v_2; w, \xi)\| \le L_i^s \|v_1 - v_2\|, \quad \forall v_1, v_2 \in \mathbb{R}^{d_i^{\mathrm{ZO}}}.
  \label{eq:ass-surrogate-smooth}
\end{equation}
By the chain rule, evaluating the surrogate gradient at the unperturbed anchor recovers the exact bottom reference gradient:
\begin{equation}
  \left. \nabla_v \mathcal{L}_i(v; w, \xi) \right|_{v = w^{\mathrm{ZO}}} = \nabla_{w^{\mathrm{ZO}}} \ell(w; \xi).
  \label{eq:surrogate-anchor-identity}
\end{equation}
\end{assumption}

\begin{assumption}[Affine Gradient Heterogeneity]
\label{ass:affine-heterogeneity}
The local data heterogeneity $H(w) \coloneqq \frac{1}{N}\sum_{i=1}^N \|\nabla f_i(w) - \nabla f(w)\|^2$ satisfies:
\begin{equation}
  H(w) \le a_{\mathrm H} \|\nabla f(w)\|^2 + \zeta^2, \quad \forall w \in \mathbb{R}^d.
  \label{eq:ass-affine}
\end{equation}
\end{assumption}

\paragraph{Discussion of assumptions.}
The smoothness and stochastic-gradient conditions are standard in non-convex federated optimization~\citep{scaffold,hosfl}. Assumption~\ref{ass:affine-heterogeneity} is equivalent to SCAFFOLD's bounded gradient dissimilarity condition (A1), with $G^2=\zeta^2$ and $B^2=1+a_{\mathrm H}$.
Assumption~\ref{ass:surrogate-new} expresses the same type of boundary-gradient and network-curvature regularity used in HO-SFL~\citep{hosfl}, directly at the scalar local-objective level: a uniformly bounded boundary gradient and a Lipschitz bottom-network Jacobian imply the stated smoothness. Client-specific $L_i^s$ accommodate different order boundary.

Throughout, $p$ is fixed before training and held constant across rounds. The filtration $\mathcal F_t$ contains $p$, $w_t$, and all history before round $t$ sampling and local training; $\mathbb E_t[\cdot]=\mathbb E[\cdot\mid\mathcal F_t]$.

\subsection{Moments of the Hybrid-Order Estimator}
\label{app:ho-moments}

Let $\varepsilon_i(w)\coloneqq\mathbb E[\widehat g_i(w)\mid w]-\nabla f_i(w)$ denote the estimator bias. For $d_i^{\mathrm{ZO}} > 0$, we define the dimensional penalty and variance coefficients:
\begin{equation}
  \begin{aligned}
  \kappa_i &\coloneqq \frac{d_i^{\mathrm{ZO}} + 1}{q}, \qquad
  \rho_i^2 \coloneqq \mu^2 (L_i^s)^2 d_i^{\mathrm{ZO}},\\
  \omega_i^2 &\coloneqq \frac{\mu^2 (L_i^s)^2}{4} d_i^{\mathrm{ZO}}(d_i^{\mathrm{ZO}}+2)(d_i^{\mathrm{ZO}}+4),\\
  \nu_i &\coloneqq 2\sigma_i^2 + 2\kappa_i \sigma_{i,\mathrm{ZO}}^2 + 2\omega_i^2.
  \end{aligned}
  \label{eq:constants-def}
\end{equation}
For full first-order clients ($d_i^{\mathrm{ZO}} = 0$), set $\kappa_i = \rho_i = \omega_i = 0$ and $\nu_i = \sigma_i^2$.

\begin{lemma}[Hybrid Estimator Moments]
\label{lem:ho-moments}
Under Assumptions~\ref{ass:reference}--\ref{ass:surrogate-new}, the hybrid-order gradient estimator $\widehat g_i(w) = [(\widehat g_i^{\mathrm{ZO}})^\top, (g_i^{\mathrm{FO}})^\top]^\top$ satisfies:
\begin{align}
  \|\mathbb{E}[\widehat g_i(w) \mid w] - \nabla f_i(w)\|^2 &\le \rho_i^2, \label{eq:est-bias-bound}\\
  \mathbb{E}\left[ \|\widehat g_i(w) - \mathbb{E}[\widehat g_i(w) \mid w]\|^2 \,\middle|\, w \right] &\le 2\kappa_i \|\Pi_i \nabla f_i(w)\|^2 + \nu_i, \label{eq:ho-energy}\\
  \mathbb{E}\left[ \|\widehat g_i(w)\|^2 \,\middle|\, w \right] &\le 2\|\nabla f_i(w)\|^2 + 2\kappa_i \|\Pi_i \nabla f_i(w)\|^2 + \nu_i. \label{eq:est-second-bound}
\end{align}
\end{lemma}

\begin{proof}
Fix model parameter $w$ and minibatch $\xi$, with independent directions $u_1,\ldots,u_q$ drawn independently of $\xi$.
Applying Taylor's theorem with integral remainder to the surrogate $\mathcal{L}_i(\cdot; w, \xi)$ along perturbation direction $u \sim \mathcal{N}(0, I_{d_i^{\mathrm{ZO}}})$:
\begin{equation}
  \frac{\mathcal{L}_i(w^{\mathrm{ZO}} + \mu u) - \mathcal{L}_i(w^{\mathrm{ZO}})}{\mu} u = u u^\top \nabla_v \mathcal{L}_i(w^{\mathrm{ZO}}) + R(u),
  \label{eq:taylor-expansion}
\end{equation}
where $R(u) \coloneqq \int_0^1 u u^\top [\nabla_v \mathcal{L}_i(w^{\mathrm{ZO}} + \tau \mu u) - \nabla_v \mathcal{L}_i(w^{\mathrm{ZO}})] \mathrm{d}\tau$.
By Assumption~\ref{ass:surrogate-new}, the remainder is bounded by:
\begin{equation}
  \|R(u)\| \le \int_0^1 \|u\|^2 L_i^s \tau \mu \|u\| \mathrm{d}\tau = \frac{\mu L_i^s}{2} \|u\|^3.
  \label{eq:remainder-bound}
\end{equation}
For Gaussian perturbations $u \sim \mathcal{N}(0, I_{d_i^{\mathrm{ZO}}})$, we have $\mathbb{E}[u u^\top] = I_{d_i^{\mathrm{ZO}}}$ and for any fixed vector $x \in \mathbb{R}^{d_i^{\mathrm{ZO}}}$:
\begin{equation}
  \mathbb{E}\left[ \left\| \frac{1}{q}\sum_{j=1}^q u_j u_j^\top x - x \right\|^2 \right] = \frac{d_i^{\mathrm{ZO}} + 1}{q} \|x\|^2 = \kappa_i \|x\|^2.
  \label{eq:gaussian-fourth-moment}
\end{equation}
Furthermore, by Jensen's inequality and the standard Gaussian moment $\mathbb{E}[\|u\|^6] = d_i^{\mathrm{ZO}}(d_i^{\mathrm{ZO}}+2)(d_i^{\mathrm{ZO}}+4)$:
\begin{equation}
  \mathbb{E}\left[ \left\| \frac{1}{q}\sum_{j=1}^q R(u_j) \right\|^2 \right] \le \frac{1}{q}\sum_{j=1}^q \mathbb{E}\|R(u_j)\|^2 \le \frac{\mu^2 (L_i^s)^2}{4} \mathbb{E}[\|u\|^6] = \omega_i^2.
  \label{eq:remainder-variance}
\end{equation}
Decompose the estimator as $\widehat g_i(w) = \widehat g_i^0(w) + [(\frac{1}{q}\sum_j R(u_j))^\top, 0^\top]^\top$, where $\widehat g_i^0(w)$ denotes the unperturbed linear component.
Taking conditional expectations over the perturbation directions:
\begin{equation}
  \mathbb{E}_u [\widehat g_i^0(w)] = \left[ (\nabla_v \mathcal{L}_i(w^{\mathrm{ZO}}))^\top, (g_i^{\mathrm{FO}})^\top \right]^\top = G_i(w, \xi).
  \label{eq:linear-unbiased}
\end{equation}
Taking expectations over minibatch $\xi$ yields $\mathbb{E}_{\xi, u}[\widehat g_i^0(w)] = \nabla f_i(w)$.
For the bias, Gaussian integration by parts gives $\mathbb{E}_u\left[\frac{\mathcal{L}_i(v+\mu u) - \mathcal{L}_i(v)}{\mu}u\right] = \mathbb{E}_u [\nabla_v \mathcal{L}_i(v + \mu u)]$.
Therefore:
\begin{equation}
  \begin{aligned}
  \|\mathbb{E}[\widehat g_i(w) \mid w] - \nabla f_i(w)\|
  &= \|\mathbb{E}_{\xi, u}[\nabla_v \mathcal{L}_i(w^{\mathrm{ZO}} + \mu u) - \nabla_v \mathcal{L}_i(w^{\mathrm{ZO}})]\|\\
  &\le \mu L_i^s \mathbb{E}\|u\| \le \mu L_i^s \sqrt{d_i^{\mathrm{ZO}}} = \rho_i,
  \end{aligned}
  \label{eq:bias-proof-step}
\end{equation}
proving~\eqref{eq:est-bias-bound}.
For variance, applying $\|x+y\|^2 \le 2\|x\|^2 + 2\|y\|^2$ to the centered estimator:
\begin{align}
  \mathbb{E}\left[ \|\widehat g_i(w) - \mathbb{E}\widehat g_i(w)\|^2 \right]
  &\le 2 \mathbb{E}\left[ \|\widehat g_i^0(w) - \nabla f_i(w)\|^2 \right] + 2 \mathbb{E}\left[ \left\| \frac{1}{q}\sum_{j=1}^q R(u_j) \right\|^2 \right] \notag\\
  &\le 2 \left( \kappa_i \mathbb{E}\|\Pi_i G_i(w, \xi)\|^2 + \mathbb{E}\|G_i(w, \xi) - \nabla f_i(w)\|^2 \right) + 2\omega_i^2 \notag\\
  &\le 2 \left( \kappa_i \|\Pi_i \nabla f_i(w)\|^2 + \kappa_i \sigma_{i,\mathrm{ZO}}^2 + \sigma_i^2 \right) + 2\omega_i^2 \notag\\
  &= 2\kappa_i \|\Pi_i \nabla f_i(w)\|^2 + \nu_i,
  \label{eq:var-proof-step}
\end{align}
which establishes~\eqref{eq:ho-energy}.
Eq.~\eqref{eq:est-second-bound} follows directly by applying $\|x+y\|^2 \le 2\|x\|^2 + 2\|y\|^2$ to the uncentered decomposition:
\begin{align*}
  \mathbb{E}\|\widehat g_i(w)\|^2
  &\le 2\mathbb{E}\|\widehat g_i^0(w)\|^2
      + 2\mathbb{E}\left\|\frac{1}{q}\sum_{j=1}^q R(u_j)\right\|^2\\
  &\le 2\|\nabla f_i(w)\|^2 + 2\kappa_i\|\Pi_i\nabla f_i(w)\|^2
      + 2\sigma_i^2 + 2\kappa_i\sigma_{i,\mathrm{ZO}}^2 + 2\omega_i^2\\
  &= 2\|\nabla f_i(w)\|^2 + 2\kappa_i\|\Pi_i\nabla f_i(w)\|^2 + \nu_i.
\end{align*}
For full first-order clients, the reference-gradient bounds give the stated moments directly.
\end{proof}

\subsection{Multi-Step Local Trajectory Analysis}
\label{app:multistep}

We analyze the local updates in~\eqref{eq:local-updates}, initialized at $w_{i,t}^0=w_t$, with a common constant step size $\eta>0$ and $E$ local steps. Each step uses a fresh minibatch and independent Gaussian directions, with independent client training streams. Let $h\coloneqq\eta E$ and define the normalized trajectory and its conditional moments:
\begin{equation}
 \begin{aligned}
 U_{i,t}&\coloneqq\frac{w_t-w_{i,t}^E}{h},\qquad
 a_{i,t}\coloneqq\mathbb E_t[U_{i,t}],\\
 v_{i,t}&\coloneqq\mathbb E_t\|U_{i,t}-a_{i,t}\|^2.
 \end{aligned}
 \label{eq:actual-update}
\end{equation}
These statistics describe the local trajectory client $i$ would execute from $w_t$, whether or not it is selected. Write
\begin{equation}
 \bar a_t\coloneqq\frac{1}{N}\sum_{i=1}^N a_{i,t},\qquad
 B_t(p)\coloneqq\sum_{i=1}^N p_i a_{i,t}-\nabla f(w_t).
 \label{eq:trajectory-aggregate-bias}
\end{equation}
Direct model averaging gives $w_{t+1}=w_t-hM_t$, where $M_t=\frac{1}{K}\sum_{i\in S_t}U_{i,t}$.

At round $t$, write $g_{i,t}=\nabla f_i(w_t)$ and $g_t=\nabla f(w_t)$. Define the local drift and average estimator bias along client $i$'s trajectory as
\begin{equation}
 \begin{aligned}
 d_{i,t}&\coloneqq\frac{1}{E}\sum_{e=0}^{E-1}\mathbb E_t[\nabla f_i(w_{i,t}^e)-g_{i,t}],\\
 \bar\varepsilon_{i,t}&\coloneqq\frac{1}{E}\sum_{e=0}^{E-1}\mathbb E_t\varepsilon_i(w_{i,t}^e).
 \end{aligned}
 \label{eq:trajectory-bias-definitions}
\end{equation}
Then $a_{i,t}=g_{i,t}+d_{i,t}+\bar\varepsilon_{i,t}$, and the aggregate bias separates three sources of error:
\begin{equation}
 B_t(p)=
 \underbrace{\sum_i(p_i-u_i)g_{i,t}}_{\text{selection skew}}
 +\underbrace{\sum_i p_i d_{i,t}}_{\text{local trajectory drift}}
 +\underbrace{\sum_i p_i\bar\varepsilon_{i,t}}_{\text{hybrid approximation bias}},
 \label{eq:bias-decomposition}
\end{equation}
where $u_i=\frac{1}{N}$. Retaining these vector sums preserves cancellation across clients~\citep{wang2023fedavg}.

Fix round index $t$ and suppress notation $t$.
Let $\tau \coloneqq \eta^2 (E-1)^2$, $Q_i \coloneqq \|g_i\|^2 + \kappa_i \|\Pi_i g_i\|^2$, and define the drift radius:
\begin{equation}
  D_i \coloneqq 2\tau (4Q_i + \nu_i).
  \label{eq:drift-radius-def}
\end{equation}

\begin{lemma}[Local Trajectory Drift and Statistics]
\label{lem:multistep}
Under Assumptions~\ref{ass:objective}--\ref{ass:surrogate-new}, if the step size satisfies $4(1+\kappa_i)L^2\tau \le \frac{1}{2}$, then for all $e \in \{0, \dots, E-1\}$:
\begin{equation}
  \mathbb{E}_t \|w_i^e - w_t\|^2 \le D_i.
  \label{eq:trajectory-bound}
\end{equation}
Furthermore, the multi-step drift vector $d_i$ and expected update $a_i$ satisfy:
\begin{equation}
  \|d_i\|^2 \le L^2 D_i, \qquad
  \|\bar\varepsilon_i\|^2 \le \rho_i^2, \qquad
  \|a_i - g_i\|^2 \le 2L^2 D_i + 2\rho_i^2,
  \label{eq:drift-mean-bounds}
\end{equation}
and the multi-step local variance satisfies:
\begin{equation}
  v_i \le 3L^2 D_i + 3\rho_i^2 + \frac{3}{E}\left( 4\kappa_i \|\Pi_i g_i\|^2 + 4\kappa_i L^2 D_i + \nu_i \right).
  \label{eq:multistep-var-bound}
\end{equation}
\end{lemma}

\begin{proof}
Let $R_i \coloneqq \max_{0 \le e < E} \mathbb{E}_t \|w_i^e - w_t\|^2$.
By Lemma~\ref{lem:ho-moments}, $L$-smoothness, and $\|\Pi_i x\| \le \|x\|$:
\begin{align}
  \mathbb{E}_t \|\widehat g_i(w_i^e)\|^2
  &\le 2\mathbb{E}_t\|\nabla f_i(w_i^e)\|^2 + 2\kappa_i \mathbb{E}_t\|\Pi_i \nabla f_i(w_i^e)\|^2 + \nu_i \notag\\
  &\le 4\|\nabla f_i(w_t)\|^2 + 4\kappa_i\|\Pi_i \nabla f_i(w_t)\|^2 + \nu_i + 4(1+\kappa_i)L^2 \mathbb{E}_t\|w_i^e - w_t\|^2 \notag\\
  &\le 4Q_i + \nu_i + 4(1+\kappa_i)L^2 R_i.
  \label{eq:local-step-second-moment}
\end{align}
Expanding the local iteration $w_i^e - w_t = -\eta \sum_{s=0}^{e-1}\widehat g_i(w_i^s)$ and applying Cauchy--Schwarz:
\begin{equation}
  \mathbb{E}_t \|w_i^e - w_t\|^2 \le \eta^2 e \sum_{s=0}^{e-1}\mathbb{E}_t \|\widehat g_i(w_i^s)\|^2 \le \eta^2 (E-1)^2 \left[ 4Q_i + \nu_i + 4(1+\kappa_i)L^2 R_i \right].
  \label{eq:cauchy-expansion}
\end{equation}
Taking the maximum over $e < E$ on both sides:
\begin{equation}
  R_i \le \tau (4Q_i + \nu_i) + 4(1+\kappa_i)L^2 \tau R_i \le \tau (4Q_i + \nu_i) + \frac{1}{2}R_i.
  \label{eq:ri-absorption}
\end{equation}
Subtracting $\frac{1}{2}R_i$ and multiplying by 2 proves~\eqref{eq:trajectory-bound}.
Next, by $L$-smoothness and Jensen's inequality:
\begin{equation}
  \|d_i\|^2 = \left\| \frac{1}{E}\sum_{e=0}^{E-1}\mathbb{E}_t[\nabla f_i(w_i^e) - g_i] \right\|^2 \le \frac{1}{E}\sum_{e=0}^{E-1} L^2 \mathbb{E}_t\|w_i^e - w_t\|^2 \le L^2 D_i.
  \label{eq:di-proof}
\end{equation}
Since $\|\varepsilon_i(w)\|^2 \le \rho_i^2$ uniformly, $\|\bar\varepsilon_i\|^2 \le \rho_i^2$.
Using $a_i - g_i = d_i + \bar\varepsilon_i$ and $\|x+y\|^2 \le 2\|x\|^2 + 2\|y\|^2$ proves~\eqref{eq:drift-mean-bounds}.

To bound the variance $v_i = \mathbb{E}_t \|U_i - a_i\|^2$, decompose $U_i = \frac{1}{E}\sum_{e=0}^{E-1}\widehat g_i(w_i^e)$ into:
\begin{equation}
  U_i = g_i + \mathcal{R}_i + \mathcal{Z}_i + \mathcal{M}_i,
  \label{eq:ui-decomp}
\end{equation}
where $\mathcal{R}_i \coloneqq \frac{1}{E}\sum_{e=0}^{E-1}(\nabla f_i(w_i^e) - g_i)$, $\mathcal{Z}_i \coloneqq \frac{1}{E}\sum_{e=0}^{E-1}\varepsilon_i(w_i^e)$, and $\mathcal{M}_i \coloneqq \frac{1}{E}\sum_{e=0}^{E-1}\mathcal{E}_i^e$ with martingale difference (where $\mathcal F_i^e$ contains $\mathcal F_t$ and the client history before drawing step $e$ randomness):
\begin{equation}
  \mathcal{E}_i^e \coloneqq \widehat g_i(w_i^e) - \mathbb{E}[\widehat g_i(w_i^e) \mid \mathcal{F}_i^e].
  \label{eq:martingale-diff}
\end{equation}
Because $\mathcal{E}_i^e$ is a martingale difference sequence ($\mathbb{E}_t[\langle \mathcal{E}_i^e, \mathcal{E}_i^s \rangle] = 0$ for $e \ne s$):
\begin{align}
  \mathbb{E}_t \|\mathcal{M}_i\|^2
  &= \frac{1}{E^2}\sum_{e=0}^{E-1}\mathbb{E}_t \|\mathcal{E}_i^e\|^2 \notag\\
  &\le \frac{1}{E^2}\sum_{e=0}^{E-1}\left( 2\kappa_i \mathbb{E}_t\|\Pi_i \nabla f_i(w_i^e)\|^2 + \nu_i \right) \notag\\
  &\le \frac{1}{E}\left( 4\kappa_i \|\Pi_i g_i\|^2 + 4\kappa_i L^2 D_i + \nu_i \right).
  \label{eq:mi-variance}
\end{align}
Since $a_i = g_i + \mathbb{E}_t\mathcal{R}_i + \mathbb{E}_t\mathcal{Z}_i$, centering yields:
$U_i - a_i = (\mathcal{R}_i - \mathbb{E}_t\mathcal{R}_i) + (\mathcal{Z}_i - \mathbb{E}_t\mathcal{Z}_i) + \mathcal{M}_i$.
Applying $\|x+y+z\|^2 \le 3(\|x\|^2 + \|y\|^2 + \|z\|^2)$:
\begin{equation}
  v_i \le 3\mathbb{E}_t\|\mathcal{R}_i\|^2 + 3\mathbb{E}_t\|\mathcal{Z}_i\|^2 + 3\mathbb{E}_t\|\mathcal{M}_i\|^2 \le 3L^2 D_i + 3\rho_i^2 + \frac{3}{E}\left( 4\kappa_i \|\Pi_i g_i\|^2 + 4\kappa_i L^2 D_i + \nu_i \right),
  \label{eq:vi-final}
\end{equation}
which concludes the proof.
\end{proof}

\subsection{Dependent Rounding and Variance Reduction}
\label{app:round-variance}

Throughout, $\operatorname{Var}_t(X) \coloneqq \mathbb E_t\|X-\mathbb E_tX\|^2$ denotes the scalar conditional variance of a vector. Let selection vector $I_t \in \{0, 1\}^N$ indicate participating clients ($I_{i,t} = \mathbb{I}(i \in S_t)$).
Algorithm~\ref{alg:depround} ensures:
\begin{equation}
  \sum_{i=1}^N I_{i,t} = K, \qquad
  \mathbb{E}[I_{i,t} \mid \mathcal{F}_t] = K p_i, \qquad
  \operatorname{Cov}(I_{i,t}, I_{j,t} \mid \mathcal{F}_t) \le 0 \quad (\forall i \ne j).
  \label{eq:depround-properties}
\end{equation}

\begin{lemma}[Variance of Averaged Local Updates]
\label{lem:round-variance}
Under dependent rounding with independent client execution, the aggregated update $M_t = \frac{1}{K}\sum_{i=1}^N I_{i,t} U_{i,t}$ satisfies:
\begin{equation}
  \operatorname{Var}(M_t \mid \mathcal{F}_t) \le \frac{1}{K}\sum_{i=1}^N p_i \left[ v_{i,t} + 2\|a_{i,t} - \bar a_t\|^2 \right].
  \label{eq:round-variance-bound}
\end{equation}
\end{lemma}

\begin{proof}
Applying the law of total variance conditioned on filtration $\mathcal{F}_t$:
\begin{equation}
  \operatorname{Var}_t(M_t) = \mathbb{E}_t \left[ \operatorname{Var}_t(M_t \mid I_t) \right] + \operatorname{Var}_t \left( \mathbb{E}_t[M_t \mid I_t] \right).
  \label{eq:total-variance}
\end{equation}
Conditional on selection $I_t$, client local trajectories are statistically independent:
\begin{equation}
  \mathbb{E}_t \left[ \operatorname{Var}_t(M_t \mid I_t) \right] = \mathbb{E}_t \left[ \frac{1}{K^2}\sum_{i=1}^N I_{i,t} v_{i,t} \right] = \frac{1}{K^2}\sum_{i=1}^N (K p_i) v_{i,t} = \frac{1}{K}\sum_{i=1}^N p_i v_{i,t}.
  \label{eq:cond-var-term}
\end{equation}
For the second term, $\mathbb{E}_t[M_t \mid I_t] = \frac{1}{K}\sum_{i=1}^N I_{i,t} a_{i,t}$.
Let $C \in \mathbb{R}^{N \times N}$ denote the covariance matrix of $I_t$ conditional on $\mathcal F_t$.
Because $\sum_i I_{i,t} = K$ is strictly constant, $\sum_{j=1}^N C_{ij} = 0$, implying $C_{ii} = \sum_{j \ne i}(-C_{ij})$.
Using negative correlation $C_{ij} \le 0$ ($i \ne j$):
\begin{align}
  \operatorname{Var}_t \left( \frac{1}{K}\sum_{i=1}^N I_{i,t} a_{i,t} \right)
  &= \frac{1}{K^2}\sum_{i=1}^N \sum_{j=1}^N C_{ij} \langle a_{i,t}, a_{j,t} \rangle \notag\\
  &= \frac{1}{2K^2}\sum_{i \ne j}(-C_{ij}) \|a_{i,t} - a_{j,t}\|^2 \notag\\
  &\le \frac{1}{K^2}\sum_{i \ne j}(-C_{ij})\left( \|a_{i,t} - \bar a_t\|^2 + \|a_{j,t} - \bar a_t\|^2 \right) \notag\\
  &= \frac{2}{K^2}\sum_{i=1}^N C_{ii} \|a_{i,t} - \bar a_t\|^2.
  \label{eq:cov-expansion}
\end{align}
Since $I_{i,t} \in \{0, 1\}$ is Bernoulli, $C_{ii} = \operatorname{Var}_t(I_{i,t}) = K p_i(1 - K p_i) \le K p_i$.
Substituting this into~\eqref{eq:cov-expansion}:
\begin{equation}
  \operatorname{Var}_t \left( \frac{1}{K}\sum_{i=1}^N I_{i,t} a_{i,t} \right) \le \frac{2}{K}\sum_{i=1}^N p_i \|a_{i,t} - \bar a_t\|^2.
  \label{eq:sampling-variance-bound}
\end{equation}
Combining~\eqref{eq:cond-var-term} and~\eqref{eq:sampling-variance-bound} completes the proof.
\end{proof}

\subsection{Multi-Step Trajectory Descent Guarantee}
\label{app:main-proof}

The following proposition applies directly to the true $E$-step updates defined in~\eqref{eq:actual-update}, before bounding their bias and variance by model dimensions and heterogeneity.

\begin{proposition}[Multi-Step Trajectory Descent Guarantee]
\label{prop:trajectory-descent}
Let Assumptions~\ref{ass:objective}--\ref{ass:surrogate-new} hold. Fix integers $1\le K\le N$ and $E,T\ge1$, and a constant local step size $\eta>0$ with $h=\eta E\le\frac{1}{L}$. Fix $p$ before training, with $\sum_i p_i=1$ and $0\le p_i\le\frac{1}{K}$. At each round, draw $S_t=\operatorname{DepRound}(Kp)$, independently of the conditionally independent client training streams, so that $|S_t|=K$ and $\Pr(i\in S_t\mid\mathcal F_t)=Kp_i$. Then, for $\Delta\coloneqq f(w_0)-f_{\inf}$,
\begin{equation}
 \begin{aligned}
 \frac{1}{T}\sum_{t=0}^{T-1}\mathbb E\|\nabla f(w_t)\|^2
 &\le\frac{2\Delta}{hT}
 +\frac{1}{T}\sum_{t=0}^{T-1}\mathbb E\|B_t(p)\|^2\\
 &\quad+\frac{Lh}{KT}\sum_{t=0}^{T-1}\mathbb E\sum_{i=1}^N p_i
 \left(v_{i,t}+2\|a_{i,t}-\bar a_t\|^2\right),
 \end{aligned}
 \label{eq:main-bound}
\end{equation}
where $\bar a_t$ and $B_t(p)$ are defined in~\eqref{eq:trajectory-aggregate-bias}.
\end{proposition}

\begin{proof}
By $L$-smoothness of global objective $f(w)$ (Assumption~\ref{ass:objective}):
\begin{equation}
  f(w_{t+1}) \le f(w_t) + \langle \nabla f(w_t), w_{t+1} - w_t \rangle + \frac{L}{2}\|w_{t+1} - w_t\|^2.
  \label{eq:smoothness-descent}
\end{equation}
Substituting $w_{t+1} - w_t = -h M_t$ and taking conditional expectations $\mathbb{E}_t[\cdot] \coloneqq \mathbb{E}[\cdot \mid \mathcal{F}_t]$:
\begin{equation}
  \mathbb{E}_t f(w_{t+1}) \le f(w_t) - h \langle g_t, \mathbb{E}_t M_t \rangle + \frac{L h^2}{2}\left( \|\mathbb{E}_t M_t\|^2 + \operatorname{Var}_t(M_t) \right).
  \label{eq:cond-descent}
\end{equation}
Recall that $\mathbb{E}_t M_t = \sum_i p_i a_{i,t} = g_t + B_t(p)$.
Using the algebraic identity $-2\langle a, b \rangle + \|b\|^2 = -\|a\|^2 + \|b - a\|^2$:
\begin{equation}
  -h \langle g_t, \mathbb{E}_t M_t \rangle + \frac{h}{2}\|\mathbb{E}_t M_t\|^2 = -\frac{h}{2}\|g_t\|^2 + \frac{h}{2}\|B_t(p)\|^2.
  \label{eq:inner-prod-identity}
\end{equation}
Since effective step size satisfies $Lh \le 1$, we have $\frac{L h^2}{2} \le \frac{h}{2}$.
Substituting~\eqref{eq:inner-prod-identity} into~\eqref{eq:cond-descent}:
\begin{equation}
  \mathbb{E}_t f(w_{t+1}) \le f(w_t) - \frac{h}{2}\|g_t\|^2 + \frac{h}{2}\|B_t(p)\|^2 + \frac{L h^2}{2}\operatorname{Var}_t(M_t).
  \label{eq:one-step-descent}
\end{equation}
Rearranging and bounding $\operatorname{Var}_t(M_t)$ via Lemma~\ref{lem:round-variance}:
\begin{equation}
  \|g_t\|^2 \le \frac{2(f(w_t) - \mathbb{E}_t f(w_{t+1}))}{h} + \|B_t(p)\|^2 + \frac{Lh}{K}\sum_{i=1}^N p_i\left( v_{i,t} + 2\|a_{i,t} - \bar a_t\|^2 \right).
  \label{eq:gt-bound}
\end{equation}
Taking total expectations, summing over $t = 0, \dots, T-1$, telescoping the objective values $\sum_{t=0}^{T-1}(\mathbb{E}f(w_t) - \mathbb{E}f(w_{t+1})) = f(w_0) - \mathbb{E}f(w_T) \le f(w_0) - f_{\inf} = \Delta$, and dividing by $T$ completes the proof of Proposition~\ref{prop:trajectory-descent}.
\end{proof}

\subsection{Formal Non-Convex Convergence Guarantee}
\label{app:affine}

The following theorem gives the formal version of Theorem~\ref{thm:convergence}, including the sampling, horizon, and parameter requirements. Let $a_{\mathrm H},\zeta\ge0$ be the constants in Assumption~\ref{ass:affine-heterogeneity}, and define
\begin{equation}
  M_p\coloneqq\sum_{i=1}^N p_i d_i^{\mathrm{ZO}}(d_i^{\mathrm{ZO}}+2)(d_i^{\mathrm{ZO}}+4),
  \qquad \chi^2(p)\coloneqq N\sum_{i=1}^N\left(p_i-\frac{1}{N}\right)^2.
  \label{eq:weighted-dimensions}
\end{equation}
Thus $M_p$ controls the finite-difference remainder, and $\chi^2(p)$ measures departure from uniform participation.

\begin{theorem}[Non-Convex Convergence of HO-FL with Fixed Sampling]
\label{thm:convergence-formal}
Suppose Assumptions~\ref{ass:objective}--\ref{ass:affine-heterogeneity} hold, with $\sigma_i^2\le\sigma^2$ and $L_i^s\le L_s$ for all clients. Fix integers $1\le K\le N$ and $E,q,T\ge1$, client boundaries, and a probability vector $p$ satisfying
$\sum_i p_i=1$ and $\frac{\epsilon}{N}\le p_i\le\frac{1}{K}$ for a fixed $\epsilon\in(0,1]$.
Use the local updates~\eqref{eq:local-updates} and the dependent-rounding and independence conditions of Proposition~\ref{prop:trajectory-descent}.
The following explicit constants are sufficient:
\begin{equation}
 \begin{aligned}
 C_0&\coloneqq 1+\left(30+\frac{12}{\epsilon}\right)
       \left[(32N+8\sigma^2)L^2+2L_s^2\right]
       +(54N+12\sigma^2)L,\\
 C_\star&\coloneqq\max\left\{
 1,\;8L^2\left(1+\frac{2N}{\epsilon}\right),\;
 \left(\frac{2NL}{\epsilon}\right)^2,\;
 144C_0^2(1+a_{\mathrm H})^2\right\}.
 \end{aligned}
 \label{eq:formal-rate-constants}
\end{equation}
Assume the sampling skew and training horizon satisfy
\begin{equation}
  a_{\mathrm H}\chi^2(p)\le\frac{1}{12},\qquad
  T\ge C_\star\left(1+\frac{1}{q}\sum_{i=1}^N p_i d_i^{\mathrm{ZO}}\right).
  \label{eq:rate-regime}
\end{equation}
Choose the common local step size and perturbation radius as
\begin{equation}
 \begin{aligned}
  \eta&=\frac{1}{E\sqrt{T\left(1+\frac{1}{q}\sum_{i=1}^N p_i d_i^{\mathrm{ZO}}\right)}},\\
  0<\mu&\le\frac{1}{\sqrt{M_p}}
       \left(\frac{1+\frac{1}{q}\sum_{i=1}^N p_i d_i^{\mathrm{ZO}}}{T}\right)^{\frac{1}{4}}
       \quad\text{if }M_p>0.
 \end{aligned}
 \label{eq:rate-stepsizes}
\end{equation}
All clients use the same $q$ and $\mu$; if $M_p=0$, all clients use FO and no perturbation radius is needed.
Then, with $\Delta=f(w_0)-f_{\inf}$ and $C_{\mathrm{rate}}\coloneqq4\Delta+(8+6\zeta^2)C_0$,
\begin{equation}
 \frac{1}{T}\sum_{t=0}^{T-1}\mathbb E\|\nabla f(w_t)\|^2
 \le C_{\mathrm{rate}}
 \sqrt{\frac{1+\frac{1}{q}\sum_{i=1}^N p_i d_i^{\mathrm{ZO}}}{T}}
 +6\zeta^2N\sum_{i=1}^N\left(p_i-\frac{1}{N}\right)^2.
 \label{eq:formal-convergence-rate}
\end{equation}
In particular, this yields the two-term $\mathcal O$ rate in~\eqref{eq:main-explicit-rate}. Its implicit constants can depend on $N,\epsilon,L,L_s,\sigma,a_{\mathrm H},\zeta$ and $\Delta$, but not on $E,T,K,q,\mu,d_i^{\mathrm{ZO}}$ or $p_i$.
\end{theorem}

\begin{proof}
We first verify that the stated horizon and parameter choices imply every stability condition used below. Let $\kappa_{\max}=\max_i\kappa_i$ and $\tau=\eta^2(E-1)^2$. Coverage and $d_i^{\mathrm{ZO}}+1\le2d_i^{\mathrm{ZO}}$ for a ZO client imply
\[
 \kappa_{\max}\le\frac{2N}{\epsilon}
 \left(1+\frac{1}{q}\sum_i p_i d_i^{\mathrm{ZO}}\right).
\]
Substituting~\eqref{eq:rate-stepsizes} and using~\eqref{eq:rate-regime} gives
\begin{equation}
 \begin{aligned}
 L\eta E&\le\frac{L}{\sqrt{C_\star}}\le1,\\
 4(1+\kappa_{\max})L^2\tau
 &\le\frac{4L^2(1+\frac{2N}{\epsilon})}{C_\star}\le\frac{1}{2},\\
 \frac{L\eta\kappa_{\max}}{K}
 &\le\frac{2NL}{\epsilon\sqrt{C_\star}}\le1.
 \end{aligned}
 \label{eq:rate-stability}
\end{equation}
Thus Proposition~\ref{prop:trajectory-descent} and Lemma~\ref{lem:multistep} apply. Define the proof coefficient
\begin{equation}
 A_p(\eta)\coloneqq
 \frac{\eta\left(E+1+\frac{1}{q}\sum_i p_i d_i^{\mathrm{ZO}}\right)}{K}
 +\eta^2(E-1)^2\left(1+\frac{1}{q}\sum_i p_i d_i^{\mathrm{ZO}}\right).
 \label{eq:rate-coefficient}
\end{equation}
Since $E,K\ge1$, the chosen step size bounds the local-drift term by $\frac{1}{T}$ and gives
\begin{equation}
 \begin{aligned}
 A_p(\eta)
 &\le2\sqrt{\frac{1+\frac{1}{q}\sum_i p_i d_i^{\mathrm{ZO}}}{T}}+\frac{1}{T}\\
 &\le3\sqrt{\frac{1+\frac{1}{q}\sum_i p_i d_i^{\mathrm{ZO}}}{T}}
 \le\frac{3}{\sqrt{C_\star}}.
 \end{aligned}
 \label{eq:rate-coefficient-bound}
\end{equation}
By~\eqref{eq:formal-rate-constants}, this ensures
$C_0(1+a_{\mathrm H})A_p(\eta)\le\frac{1}{4}$.

Next, fix a round and suppress $t$. Write $g=\nabla f(w_t)$, $g_i=\nabla f_i(w_t)$, $H=H(w_t)$, $G_2=\|g\|^2+H$, and $V=\operatorname{Var}(M_t\mid\mathcal F_t)$. Using $Q_i,D_i$ from Lemma~\ref{lem:multistep}, define $Q_p=\sum_i p_iQ_i$, $\overline D=\sum_i p_iD_i$, $\nu_p=\sum_i p_i\nu_i$ and $\rho_p^2=\sum_i p_i\rho_i^2$.
Since $\|g_i-g\|^2\le NH$ and $\sum_i p_i(1+\kappa_i)\le2(1+\frac{1}{q}\sum_i p_i d_i^{\mathrm{ZO}})$, the estimator moments imply
\begin{align}
 Q_p&\le4N\left(1+\frac{1}{q}\sum_i p_i d_i^{\mathrm{ZO}}\right)G_2,\notag\\
 \nu_p&\le4\sigma^2\left(1+\frac{1}{q}\sum_i p_i d_i^{\mathrm{ZO}}\right)+\frac{L_s^2}{2}\mu^2M_p,
 \qquad \rho_p^2\le L_s^2\mu^2M_p,\notag\\
 \overline D&=2\tau(4Q_p+\nu_p).
 \label{eq:weighted-moment-bounds}
\end{align}
Cauchy--Schwarz gives
\begin{equation}
 \left\|\sum_i\left(p_i-\frac{1}{N}\right)g_i\right\|^2\le\chi^2(p)H.
 \label{eq:chi-heterogeneity}
\end{equation}
Combining this with the bias decomposition and Lemma~\ref{lem:multistep} yields
\begin{equation}
 \|B(p)\|^2\le3\chi^2(p)H+3L^2\overline D+3\rho_p^2.
 \label{eq:bp-sq-bound}
\end{equation}
For the participation variance, let $r_i=a_i-g_i$ and $\bar r=\frac{1}{N}\sum_i r_i$. The coverage condition gives $\|\bar r\|^2\le\frac{1}{\epsilon}\sum_i p_i\|r_i\|^2$, so
\begin{equation}
 \sum_i p_i\|a_i-\bar a\|^2
 \le3NH+6\left(1+\frac{1}{\epsilon}\right)(L^2\overline D+\rho_p^2).
 \label{eq:weighted-update-dispersion}
\end{equation}
Together with~\eqref{eq:multistep-var-bound}, Lemma~\ref{lem:round-variance}, and the verified stability bounds~\eqref{eq:rate-stability}, this gives
\begin{equation}
 \begin{aligned}
 \|B(p)\|^2+LhV
 &\le3\chi^2(p)H
 +\left(30+\frac{12}{\epsilon}\right)L^2\overline D
 +\left(18+\frac{12}{\epsilon}\right)\rho_p^2\\
 &\quad+\frac{6NL\eta E}{K}H
 +\frac{12L\eta}{K}Q_p+\frac{3L\eta}{K}\nu_p.
 \end{aligned}
 \label{eq:aggregate-error-expanded}
\end{equation}
Substitute~\eqref{eq:weighted-moment-bounds}. Since~\eqref{eq:rate-stability} also implies $L^2\tau\le\frac{1}{8}$ and $\frac{L\eta}{K}\le1$, the constant $C_0$ in~\eqref{eq:formal-rate-constants} bounds all resulting coefficients, yielding
\begin{equation}
 \|B(p)\|^2+LhV
 \le3\chi^2(p)H+C_0A_p(\eta)(\|g\|^2+H+1)+C_0\mu^2M_p.
 \label{eq:affine-step-bound}
\end{equation}
For all-FO training, the terms involving $\mu^2M_p$ are understood as zero.

By~\eqref{eq:rate-regime}, $3a_{\mathrm H}\chi^2(p)\le\frac{1}{4}$, and the first part of the proof already established $C_0(1+a_{\mathrm H})A_p(\eta)\le\frac{1}{4}$. Substituting the affine heterogeneity bound into~\eqref{eq:affine-step-bound} therefore contributes at most $\frac{1}{2}\|g\|^2$ to the right-hand side of Proposition~\ref{prop:trajectory-descent}. Absorbing this term gives
\begin{equation}
 \frac{1}{T}\sum_{t=0}^{T-1}\mathbb E\|g_t\|^2
 \le\frac{4\Delta}{\eta ET}
 +2C_0(1+\zeta^2)A_p(\eta)+2C_0\mu^2M_p+6\zeta^2\chi^2(p).
 \label{eq:rate-before-balancing}
\end{equation}
Finally, $\frac{1}{\eta ET}=\sqrt{(1+\frac{1}{q}\sum_i p_i d_i^{\mathrm{ZO}})/T}$, while~\eqref{eq:rate-stepsizes} bounds $\mu^2M_p$ by the same quantity. Substitution into~\eqref{eq:rate-before-balancing}, together with~\eqref{eq:rate-coefficient-bound}, gives~\eqref{eq:formal-convergence-rate} with $C_{\mathrm{rate}}=4\Delta+(8+6\zeta^2)C_0$.
\end{proof}

The step size controls local drift directly, yielding a rate that isolates dimension and sampling effects without claiming speedup in $E$ or $K$. Uniform participation removes the sampling-bias term; non-uniform participation trades a smaller sampling-weighted ZO dimension against residual data heterogeneity.

% ===== END appendices/theory_v3.tex =====

\section{Dimension-Aware Sampling: Derivations and Practical Guidance}
\label{app:estimation}

This appendix supplies the derivations behind Section~\ref{sec:sampling} and a practical procedure for choosing the fixed sampling preference $\beta$. We first derive the scalar relaxation and its constrained solution, then quantify the role of $\beta$ and describe how to select it before training.

\subsection{Derivation of the Sampling Relaxation}
\label{app:qp}

\paragraph{The statistical reference problem.}
Condition on the history before a round, and treat $a_i,v_i,g$ as fixed conditional statistics. Set $\mathcal A=[a_1,\ldots,a_N]$ and $c_i=\frac{Lh}{K}[v_i+2\|a_i-\bar a\|^2]$. The sampling-dependent part of Proposition~\ref{prop:trajectory-descent} is
\begin{equation}
 J(p)=\|\mathcal A p-g\|^2+c^\top p
 =p^\top\mathcal A^\top\mathcal A p-2g^\top\mathcal A p+\|g\|^2+c^\top p.
 \label{eq:app-sampling-objective}
\end{equation}
Its Hessian is $2\mathcal A^\top\mathcal A\succeq0$, establishing convexity of~\eqref{eq:sampling-qp}. This is an optimum of the current conditional bound; the practical policy uses a single probability vector fixed before training.

\paragraph{Scalar relaxation.}
Let $u_i=\frac{1}{N}$ and $H_a=\frac{1}{N}\sum_i\|a_i-\bar a\|^2$. Since $\sum_i(p_i-u_i)=0$,
\begin{align}
 \mathcal A p-g
 &=\sum_i(p_i-u_i)(a_i-\bar a)+(\bar a-g),\notag\\
 \|\mathcal A p-g\|^2
 &\le2\left\|\sum_i(p_i-u_i)(a_i-\bar a)\right\|^2+2\|\bar a-g\|^2\notag\\
 &\le2\|p-u\|^2\sum_i\|a_i-\bar a\|^2+2\|\bar a-g\|^2.
 \label{eq:app-scalar-relaxation}
\end{align}
Thus $J(p)\le A\|p-u\|^2+c^\top p+2\|\bar a-g\|^2$, with $A=2NH_a$. The last term is independent of $p$, yielding~\eqref{eq:scalar-qp}. This removes $g$ from the sampling optimization, while its effect remains in the convergence analysis through $\bar a-g$. The relaxation preserves client-level noise and update dispersion but discards directional cancellation and the cross term with $\bar a-g$.

\paragraph{KKT conditions and the clipped solution.}
For $A>0$, first consider $K<N$ and $\epsilon<1$, and write $\underline p=\frac{\epsilon}{N}$ and $\overline p=\frac{1}{K}$. The Lagrangian is
\[
 \mathcal J(p,\lambda,\alpha,\omega)
 =A\|p-u\|^2+c^\top p+\lambda\Big(\sum_i p_i-1\Big)
 +\sum_i\alpha_i(\underline p-p_i)+\sum_i\omega_i(p_i-\overline p),
\]
where $\alpha_i,\omega_i\ge0$. Stationarity and complementary slackness require
\begin{equation}
 \begin{aligned}
 2A(p_i-u_i)+c_i+\lambda-\alpha_i+\omega_i&=0,\\
 \alpha_i(p_i-\underline p)=0,\qquad
 \omega_i(\overline p-p_i)&=0.
 \end{aligned}
 \label{eq:app-sampling-kkt}
\end{equation}
At an interior coordinate, both inequality multipliers vanish, giving $p_i=u_i-\frac{c_i+\lambda}{2A}$. At the lower bound, stationarity implies that this unconstrained value is at most $\underline p$; at the upper bound it is at least $\overline p$. Consequently,
\begin{equation}
 p_i^\star=\operatorname{clip}_{[\epsilon/N,\,1/K]}
 \left(\frac{1}{N}-\frac{c_i+\lambda}{2A}\right).
 \label{eq:clip-solution}
\end{equation}
Strict convexity makes $p^\star$ unique. The sum of its coordinates is continuous and nonincreasing in $\lambda$. A valid bisection bracket is
\[
 \lambda_{\mathrm{lo}}=\min_i\left\{2A\left(\frac{1}{N}-\frac{1}{K}\right)-c_i\right\},\qquad
 \lambda_{\mathrm{hi}}=\max_i\left\{2A\left(\frac{1}{N}-\frac{\epsilon}{N}\right)-c_i\right\}.
\]
At these endpoints, the coordinate sums are $\frac{N}{K}\ge1$ and $\epsilon\le1$, respectively. When $A=0$, the objective is linear: initialize every coordinate at $\frac{\epsilon}{N}$ and allocate the remaining mass in increasing order of $c_i$, up to $\frac{1}{K}$; tied costs can share the allocation. If $K=N$ or $\epsilon=1$, the feasible set is simply $\{u\}$.

\subsection{Choosing $\beta$ in Practice}
\label{app:beta-practice}

HO-FL uses $\beta$ as a single preference parameter, selected before the training run and held fixed. The following guidelines combine scenario knowledge, a small validation budget, and desired participation coverage.

\paragraph{Starting from scenario knowledge.}
When client data are broadly representative of the same population and ZO dimensions differ substantially, $\beta=0.75$ is a useful starting point. For intermediate heterogeneity, or limited prior knowledge, start from $\beta=0.5$. With pronounced client-specific data skew, start from $\beta=0.25$; $\beta=0$ remains the uniform reference. Resource--data correlation matters as well: if the preferred clients represent only a narrow part of the population, a smaller preference is appropriate even when their dimension advantage is large. These values are empirical starting points, rather than universal heterogeneity thresholds. Table~\ref{tab:beta-v2} illustrates the pattern on Qwen/AG News: $0.75$ performs best under IID and the two milder skews, whereas $0.25$ performs best at $\alpha=0.1$. Further IID comparisons appear in Appendix~\ref{app:iid-beta}.

\paragraph{A small-budget validation procedure.}
Without reliable scenario knowledge, we recommend comparing $\beta\in\{0,0.25,0.5,0.75\}$ in short pilot runs on a representative subset of the target task. Preserve the client partition structure, resource tiers, model boundaries, and any resource--data correlation. Evaluate on held-out validation data spanning all resource tiers, using the mean client validation loss to match the uniformly weighted client objective. Select the best candidate, favoring a smaller $\beta$ when validation differences are inconclusive, then fix it for the formal run. Federated validation and such successive-halving procedures are established approaches to hyperparameter selection~\citep{khodak2021fedex}. This is a recommended deployment procedure; the experimental configurations and sensitivity sweeps used in this paper are reported in Appendix~\ref{app:experiments}.

\paragraph{Choosing a coverage budget.}
A complementary criterion requires no gradient or heterogeneity estimates. For a nonpreferred client $i\notin R$, exact inclusion probabilities imply
\begin{equation}
 \mathbb E\!\left[\sum_{t=0}^{T-1}\mathbf1\{i\in S_t\}\right]
 =TKp_i=\frac{KT}{N}(1-\beta).
 \label{eq:app-beta-coverage}
\end{equation}
If the desired minimum expected participation count is $m\in[0,KT/N]$, restrict candidates to
\begin{equation}
 \beta\le1-\frac{mN}{KT}.
 \label{eq:app-beta-coverage-cap}
\end{equation}
For our $N=30$, $K=6$, and $T=160$ setting, $\beta=0.25,0.5,0.75$ correspond to $24,16,8$ expected participations per nonpreferred client. This is an expectation, rather than a guaranteed count in every run. It translates the preference parameter into a concrete participation budget: choose an acceptable coverage level, then use scenario knowledge or validation to select among the remaining candidates. Once selected, $\beta$ and the resource ranking remain fixed, so this procedure adds no online moment estimation to the training algorithm.

\section{Experimental Details}
\label{app:experiments}

% ===== BEGIN appendices/experiments_v2.tex =====

\subsection{Datasets and Benchmark Configurations}
\label{app:datasets}
We evaluate all frameworks across five representative language understanding and generation benchmarks. Table~\ref{tab:dataset-stats} summarizes their task types, sample sizes, and evaluation metrics:
\begin{itemize}[leftmargin=18pt]
  \item \textbf{SST-2}~\citep{sst2}: The Stanford Sentiment Treebank binary classification task. Evaluated on the full official validation set (872 examples).
  \item \textbf{BoolQ}~\citep{boolq}: A reading comprehension benchmark of natural yes/no questions paired with Wikipedia passages. Evaluated on the full official validation set (3,270 examples).
  \item \textbf{SciQ}~\citep{sciq}: A crowdsourced multiple-choice science examination benchmark spanning physics, chemistry, and biology. Evaluated on the official test set (1,000 examples).
  \item \textbf{AG News}~\citep{zhang2015character}: A 4-class topic classification dataset (World, Sports, Business, Sci/Tech). We evaluate IID and non-IID settings with 4,000 training examples per client to control for quantity skew~\citep{li2022noniid}. Non-IID class proportions follow
$\operatorname{Dir}(\alpha\mathbf{1}_4)$,
$\alpha \in \{1, 0.5, 0.1\}$; exhausted class pools are replenished by sampling with replacement, allowing duplicates within and across clients. Evaluation uses the standard test set (7,600 examples).
  \item \textbf{SQuAD 1.1}~\citep{squad}: A challenging token-level extractive question answering benchmark. We adopt greedy auto-regressive decoding with a maximum prompt length of 512 tokens and a generation cap of 64 tokens. We report official Exact Match (EM) and Macro-F1 metrics on all 10,570 validation examples.
\end{itemize}

\begin{table}[htbp]
\centering
\caption{\textbf{Benchmark dataset statistics and evaluation protocols.}}
\label{tab:dataset-stats}
\small
\begin{tabular}{llrrll}
\toprule
Dataset & Task & Train & Eval & Partition & Metric \\
\midrule
SST-2 & Sentiment Classification & 67,349 & 872 & IID & Accuracy (\%) \\
BoolQ & Reading Comprehension (Y/N) & 9,427 & 3,270 & IID & Accuracy (\%) \\
SciQ & Science QA (4 choices) & 11,679 & 1,000 & IID & Accuracy (\%) \\
AG News & Topic Classification (4 classes) & 120,000 & 7,600 & IID / $\operatorname{Dir}(\alpha)$ & Accuracy (\%) \\
SQuAD 1.1 & Extractive QA & 87,599 & 10,570 & IID & EM / F1 (\%) \\
\bottomrule
\end{tabular}
\end{table}

\subsection{Model Architectures and LoRA Configuration}
\label{app:model-configs}
All pretrained foundation models are instantiated using their official HuggingFace checkpoints: \texttt{facebook/opt-125m} (12 transformer layers, hidden dimension 768), \texttt{Qwen/Qwen2.5-1.5B} (28 layers, hidden dimension 1536), and \texttt{HuggingFaceTB/SmolLM3-3B} (36 layers, hidden dimension 2048).
The backbone network weights are fully frozen and cast to BF16 precision.
Low-Rank Adaptation (LoRA)~\citep{lora} is injected exclusively into the query and value projection matrices ($W_q, W_v$) across all attention modules:
\begin{itemize}[leftmargin=18pt]
  \item \textbf{LoRA Hyperparameters:} We configure adapter rank $r=8$, LoRA scaling factor $\alpha_{\mathrm{lora}}=16$, and zero dropout across all experiments. LoRA matrices $A$ and $B$ are stored in FP32 precision to guarantee numerical stability during zeroth-order gradient perturbations and finite-difference evaluations.
  \item \textbf{Optimizer Configuration:} Local training uses AdamW~\citep{adamw} with decoupled weight decay $5 \times 10^{-4}$, learning rate $10^{-5}$, momentum parameters $(\beta_1, \beta_2) = (0.9, 0.999)$, and stabilizer $\epsilon = 10^{-8}$. Optimizer momentum buffers are reset upon each client participation.
\end{itemize}

\subsection{Baseline Implementation Details}
\label{app:baselines}
We compare HO-FL against the following state-of-the-art baselines using matched data partitions and seeds at the same number of communication rounds:
\begin{itemize}[leftmargin=18pt]
  \item \textbf{FedAvg}~\citep{fedavg}: Executes standard backpropagation (FO) on all trainable LoRA adapters across the entire model. All selected clients update all adapters and transmit their parameter deltas.
  \item \textbf{FedProx}~\citep{fedprox}: Adds a proximal regularization term $\frac{\mu_{\mathrm{prox}}}{2}\|w - w_t\|^2$ to the local objective to combat client drift under non-IID data distributions. We set $\mu_{\mathrm{prox}} = 10^{-3}$.
  \item \textbf{DeComFL}~\citep{decomfl}: A state-of-the-art pure zeroth-order federated optimization method. All trainable parameters are perturbed via randomized Gaussian directions ($q=2, \mu=10^{-3}$). Communication cost is dimension-free via shared seed generation. We adapt DeComFL to support local AdamW updates with deterministic replay.
  \item \textbf{HO-FL ($\beta=0$):} Executes hybrid-order local updates identical to HO-FL across the heterogeneous client tiers, but samples participating clients uniformly at random ($p_i = 1/N$, equivalent to setting $\beta=0$).
  \item \textbf{HO-FL (Proposed):} Employs the dimension-aware client sampling policy $p = (1-\beta)u + \beta r$ with dependent rounding $S_t = \operatorname{DepRound}(K p)$, using $\beta=0.75$ for IID benchmarks and exploring $\beta \in [0, 0.75]$ on Dirichlet label partitions.
\end{itemize}

\paragraph{Training and evaluation.}
All models use AdamW with learning rate $10^{-5}$, moment coefficients $(0.9,0.999)$, numerical stabilizer $10^{-8}$, and weight decay $5\times10^{-4}$; optimizer states reset at each client participation. FedProx uses a proximal coefficient of $10^{-3}$. The frozen backbone uses BF16, and the query/value LoRA adapters use FP32 with rank 8, scaling parameter 16, and zero dropout. The default configuration has $N=30$, $K=6$, batch size 16, $E=5$, $q=2$, $\mu=10^{-3}$, and 160 rounds.

SST-2 and BoolQ use their full validation sets of 872 and 3,270 examples. SciQ uses the official 1,000-example test set. SQuAD uses greedy generation with a prompt cap of 512 tokens and an output cap of 64 tokens; Table~\ref{tab:squad-v2} reports official v1.1 EM/F1 on all 10,570 validation examples. Intermediate SQuAD curves use a fixed 1,024-example subset. We report the final model throughout.

\paragraph{Seeds and sampling.}
All entries in Table~\ref{tab:llm-v2}, the Qwen/AG News sensitivity study, and the local-training ablations use three seeds, reporting means and sample SDs. All SQuAD results use these three seeds. All SQuAD sampling comparisons use three matched seeds. Within each seed, comparisons preserve the client data, initialization, and resource assignment. The preferred set $R$ is fixed using the smallest actual trainable ZO dimensions divided by $q$, with an independent seeded tie-breaking permutation. HO-FL applies dependent rounding to $Kp$ and directly averages the selected models.

\subsection{IID sampling preference}
\label{app:iid-beta}
Table~\ref{tab:iid-beta-details} compares sampling strengths using seed 42; Table~\ref{tab:squad-beta-details} uses three seeds. Each comparison preserves the configuration within each seed. We use $\beta=0.75$ for the IID main results. Relative to $0.5$, it preserves SST-2 accuracy for both OPT and Qwen, preserves OPT/BoolQ accuracy, and improves Qwen/BoolQ and Qwen/SciQ by 0.55 and 0.30 points. OPT/SciQ decreases by 0.20 points while remaining near chance. On SQuAD, $0.75$ improves both EM and F1 over $0.5$ for all three models in every seed; the mean F1 gains are 1.27, 0.30, and 0.27 points for OPT, Qwen, and SmolLM3. The three-seed Qwen/AG News IID comparison in Table~\ref{tab:beta-v2} also favors $0.75$ over $0.5$ in every seed.
\begin{table}[htbp]
\centering
\caption{\textbf{IID sampling sensitivity.} Final accuracy (\%). All columns use HO-FL.}
\label{tab:iid-beta-details}
\small
\begin{tabular}{llcccc}
\toprule
Model & Task & $\beta=0$ & $0.25$ & $0.5$ & $0.75$ \\
\midrule
OPT & SST-2 & 86.01 & 86.47 & 86.47 & 86.47 \\
 & BoolQ & 57.98 & 57.92 & 58.84 & 58.84 \\
 & SciQ & 24.20 & 23.90 & 24.50 & 24.30 \\
\midrule
Qwen & SST-2 & 91.28 & 91.86 & 92.78 & 92.78 \\
 & BoolQ & 76.54 & 77.74 & 78.29 & 78.84 \\
 & SciQ & 89.90 & 89.80 & 90.00 & 90.30 \\
\bottomrule
\end{tabular}
\end{table}

\begin{table}[htbp]
\centering
\caption{\textbf{SQuAD sampling sensitivity.} Final full-validation EM/F1 (\%), reported as $\text{mean}_{(\text{SD})}$ accross three seeds.}
\label{tab:squad-beta-details}
\small
\begin{tabular}{llccc}
\toprule
Model & Metric & $\beta=0$ & $0.5$ & $0.75$ \\
\midrule
OPT & EM & $37.16_{(0.26)}$ & $40.68_{(0.23)}$ & $41.87_{(0.32)}$ \\
 & F1 & $47.24_{(0.21)}$ & $51.16_{(0.19)}$ & $52.43_{(0.20)}$ \\
\midrule
Qwen & EM & $76.67_{(0.05)}$ & $78.32_{(0.18)}$ & $78.71_{(0.19)}$ \\
 & F1 & $85.34_{(0.03)}$ & $86.61_{(0.11)}$ & $86.92_{(0.10)}$ \\
\midrule
Smol & EM & $77.80_{(0.11)}$ & $79.23_{(0.06)}$ & $79.54_{(0.20)}$ \\
 & F1 & $87.25_{(0.04)}$ & $88.30_{(0.06)}$ & $88.57_{(0.13)}$ \\
\bottomrule
\end{tabular}
\end{table}

Stronger preference increases the participation of clients with larger FO-trained segments. With three equally sized resource tiers, their expected per-round counts are $4/1/1$ at $\beta=0.5$ and $5/0.5/0.5$ at $0.75$, ordered from the smallest to the largest ZO segment. The choice of $0.75$ therefore favors task performance while retaining the per-device memory benefit of hybrid training. The heterogeneity study reports the full $\beta$ grid, and the local-training ablations below keep their fixed $\beta=0.5$ setting.

\FloatBarrier
\subsection{Memory profiles and local participation simulation}
\label{app:memory-v2}
Profiles use Qwen2.5-1.5B, batch size 16, $E=5$, and $q=2$, with FP32 query/value LoRA adapters and a frozen BF16 backbone. We use the maximum peak allocated bytes from three isolated measurements of a complete local participation, including optimizer state initialization. Only sequence length 1024 is displayed; FedProx is omitted from the memory comparison.

The right panel of Figure~\ref{fig:memory-v2} is generated locally, without training a model. It retains the seed-42 client cut assignment and preferred set from the Qwen experiments, with ten clients each at cuts 8/16/24. We set $p_i=(1-\beta)/30+\beta\mathbf1\{i\in R\}/6$ and simulate 160 rounds of fixed-size dependent rounding for each $\beta\in\{0,0.25,0.5,0.75\}$. Each round uses a fresh random processing order and randomness derived from the simulation seed and round index. Client profiles are summed for the six selected participants and averaged over rounds. The corresponding expectation is $6\sum_i p_i m_i(L)$; finite-round averages need not equal this expectation exactly. FedAvg and DeComFL use homogeneous profiles, so their round totals are six times their single-client requirement.

\FloatBarrier
\subsection{Local training ablations}
\label{app:ablations-v2}
We vary one parameter at a time on OPT-125M/SST-2 IID with HO-FL, fixing $\beta=0.5$ and using three random seeds. The default $E=5$, $q=2$, $\mu=10^{-3}$ configuration is shared across the sweeps. Table~\ref{tab:local-sensitivity} reports final full-validation accuracy and task loss.
\begin{table}[htbp]
\centering
\caption{\textbf{Local-training sensitivity on OPT/SST-2 IID.} HO-FL with $\beta=0.5$; mean $\pm$ sample SD across three seeds. All configurations process 76,800 training examples. The default is shared across the three sweeps.}
\label{tab:local-sensitivity}
\small
\begin{tabular}{lrrrrcc}
\toprule
Configuration & $E$ & $q$ & $\mu$ & Rounds & Accuracy (\%) & Validation loss \\
\midrule
Default & 5 & 2 & $10^{-3}$ & 160 & $86.39 \pm 0.07$ & $0.3498 \pm 0.0038$ \\
$E=1$ & 1 & 2 & $10^{-3}$ & 800 & $86.47 \pm 0.41$ & $0.3489 \pm 0.0033$ \\
$E=10$ & 10 & 2 & $10^{-3}$ & 80 & $86.58 \pm 0.30$ & $0.3406 \pm 0.0008$ \\
$q=1$ & 5 & 1 & $10^{-3}$ & 160 & $86.51 \pm 0.18$ & $0.3498 \pm 0.0036$ \\
$q=4$ & 5 & 4 & $10^{-3}$ & 160 & $86.47 \pm 0.00$ & $0.3498 \pm 0.0038$ \\
$q=8$ & 5 & 8 & $10^{-3}$ & 160 & $86.47 \pm 0.11$ & $0.3496 \pm 0.0041$ \\
$\mu=10^{-4}$ & 5 & 2 & $10^{-4}$ & 160 & $86.31 \pm 0.07$ & $0.3496 \pm 0.0036$ \\
$\mu=10^{-2}$ & 5 & 2 & $10^{-2}$ & 160 & $86.35 \pm 0.00$ & $0.3493 \pm 0.0040$ \\
\bottomrule
\end{tabular}
\end{table}

\paragraph{Local steps $E$.}
For $E\in\{1,5,10\}$, we use 800, 160, and 80 rounds, respectively, so each run processes 76,800 training examples. Mean accuracy is 86.47\%, 86.39\%, and 86.58\%. In this range, increasing local work preserves performance while reducing the number of communication rounds; $E=10$ also yields a lower final validation loss.

\paragraph{Direction count $q$.}
For $q\in\{1,2,4,8\}$, mean accuracy is 86.51\%, 86.39\%, 86.47\%, and 86.47\%. Increasing the direction count does not produce a monotonic performance gain in this setting. A small number of directions is sufficient, supporting the default $q=2$ with its lower local computation cost. 

\paragraph{Perturbation radius $\mu$.}
The radii $10^{-4}$, $10^{-3}$, and $10^{-2}$ yield 86.31\%, 86.39\%, and 86.35\% mean accuracy. The small variation across this range supports retaining $\mu=10^{-3}$. Together, these sweeps show that performance is stable across the tested local-training parameters.

\stopcontents[appendices]
\end{document}